\documentclass[11pt]{article}
\usepackage{hyperref}
\usepackage{yub}
\usepackage{paralist}

\usepackage{microtype}
\usepackage{graphicx}
\usepackage{subfigure}
\usepackage{booktabs} 

\usepackage[top=1in, bottom=1in, left=1in, right=1in]{geometry}
\usepackage[numbers]{natbib}

\usepackage{algorithm}
\usepackage{algorithmic}

\usepackage{mathtools}
\mathtoolsset{showonlyrefs}

\def\shownotes{0}  
\ifnum\shownotes=1
\newcommand{\authnote}[2]{$\ll$\textsf{\footnotesize #1 notes: #2}$\gg$}
\else
\newcommand{\authnote}[2]{}
\fi
\newcommand{\yub}[1]{{\color{red}\authnote{Yu}{#1}}}

\newcommand{\tx}[1]{{\color{violet}\authnote{Tengyang}{#1}}}

\renewcommand{\setto}{\leftarrow}
\newcommand{\last}{{\rm last}}

\newcommand{\Reg}{{\rm Regret}}
\newcommand{\sw}{{\rm switch}}

\makeatletter
\def\blfootnote{\gdef\@thefnmark{}\@footnotetext}
\makeatother

\usepackage{amsthm}

\makeatletter
\newtheorem*{rep@theorem}{\rep@title}
\newcommand{\newreptheorem}[2]{%
\newenvironment{rep#1}[1]{%
 \def\rep@title{#2 \ref{##1}}%
 \begin{rep@theorem}}%
 {\end{rep@theorem}}}
\makeatother

\newreptheorem{theorem}{Theorem}

\usepackage{mathtools}
\mathtoolsset{showonlyrefs}

\hypersetup{
    colorlinks,
    linkcolor={blue!50!black},
    citecolor={blue!50!black},
}
\colorlet{linkequation}{blue}

\title{Provably Efficient Q-Learning with Low Switching Cost}
\author{Yu Bai
  \thanks{Stanford University. \texttt{yub@stanford.edu}}
  \and Tengyang Xie
  \thanks{University of Illinois at
    Urbana-Champaign. \texttt{tx10@illinois.edu}}
  \and Nan Jiang
  \thanks{University of Illinois at
    Urbana-Champaign. \texttt{nanjiang@illinois.edu}}
  \and
  Yu-Xiang Wang
  \thanks{UC Santa Barbara. \texttt{yuxiangw@cs.ucsb.edu}}
}
\date{\today}

\begin{document}
\maketitle

\begin{abstract}
  We take initial steps in studying PAC-MDP algorithms with limited
  adaptivity, that is, algorithms that change its exploration policy
  as infrequently as possible during regret minimization.  This is
  motivated by the difficulty of running fully adaptive algorithms in
  real-world applications (such as medical domains), and we propose to
  quantify adaptivity using the notion of \emph{local switching cost}.
  Our main contribution, Q-Learning with UCB2 exploration, is a
  model-free algorithm for $H$-step episodic MDP that achieves
  sublinear regret whose local switching cost in $K$ episodes is
  $O(H^3SA\log K)$, and we provide a lower bound of $\Omega(HSA)$ on
  the local switching cost for any no-regret algorithm. Our algorithm
  can be naturally adapted to the concurrent setting
  \citep{guo2015concurrent}, which yields nontrivial results that
  improve upon prior work in certain aspects.
\end{abstract}
\section{Introduction}

This paper is concerned with reinforcement learning (RL) under
\emph{limited adaptivity} or \emph{low switching cost}, a setting in
which the agent is allowed to act in the environment for a long period
but is constrained to switch its policy for at most $N$ times. A small
switching cost $N$ restricts the agent from frequently adjusting its
exploration strategy based on feedback from the environment. 

There are strong practical motivations for developing RL algorithms
under limited adaptivity. The setting of restricted policy switching
captures various real-world settings where deploying new policies
comes at a cost. For example, in medical applications where actions
correspond to treatments, it is often unrealistic to execute fully
adaptive RL algorithms -- instead one can only run a fixed policy
approved by the domain experts to collect data, and a separate
approval process is required every time one would like to switch to a
new policy~\citep{lei2012smart, almirall2012designing,
  almirall2014introduction}.  In personalized
recommendation~\cite{theocharous2015personalized}, it is
computationally impractical to adjust the policy online based on
instantaneous data (for instance, think about online video
recommendation where there are millions of users generating feedback
at every second). A more common practice is to aggregate data in a
long period before deploying a new policy. In problems where we run RL
for compiler optimization \citep{ashouri2018survey} and hardware
placements \citep{mirhoseini2017device}, as well as for learning to
optimize databases~\citep{krishnan2018learning}, often it is desirable
to limit the frequency of changes to the policy since it is costly to
recompile the code, to run profiling, to reconfigure an FPGA devices,
or to restructure a deployed relational database. The problem is even
more prominent in the RL-guided new material discovery as it takes
time to fabricate the materials and setup the experiments
\citep{raccuglia2016machine,nguyen2019incomplete}. In many of these
applications, adaptivity turns out to be really the bottleneck.

Understanding limited adaptivity RL is also important from a
theoretical perspective.  First, algorithms with low adaptivity
(a.k.a.  ``batched'' algorithms) that are as effective as their fully
sequential counterparts have been established in
bandits~\cite{perchet2016batched,gao2019batched}, online
learning~\cite{cesa2013online}, and
optimization~\cite{duchi2018minimax}, and it would be interesting to
extend such undertanding into RL. Second, algorithms with few policy
switches are naturally easy to parallelize as there is no need for
parallel agents to communicate if they just execute the same
policy. Third, limited adaptivity is closed related to off-policy
RL\footnote{In particular, $N=0$ corresponds to off-policy RL, where
  the algorithm can only choose one data collection policy
  \citep{hanna2017data}.} and offers a relaxation less challenging
than the pure off-policy setting. We would also like to note that
limited adaptivity can be viewed as a constraint for designing RL
algorithms, which is conceptually similar to those in constrained
MDPs~\cite{chow2018lyapunov,yu2019convergent}.






In this paper, we take initial steps towards studying theoretical
aspects of limited adaptivity RL through designing \emph{low-regret
  algorithms} with limited adaptivity. We focus on model-free
algorithms, in particular Q-Learning, which was recently shown to
achieve a $\wt{O}(\sqrt{{\rm poly}(H)\cdot SAT})$ regret bound with
UCB exploration and a careful stepsize choice by~\citet{jin2018q}. Our
goal is to design Q-Learning type algorithms that achieve similar
regret bounds with a bounded switching cost.


The main contributions of this paper are summarized as follows:
\begin{compactenum}[\textbullet]
	\setlength{\itemsep}{5pt}
\item We propose a notion of \emph{local switching cost} that captures
  the adaptivity of an RL algorithm in episodic MDPs
  (Section~\ref{section:problem-setup}). Algorithms with lower local
  switching cost will make fewer switches in its
  deployed policies. 
\item Building on insights from the UCB2 algorithm in multi-armed
  bandits~\cite{auer2002finite} (Section~\ref{section:ucb2-bandits}),
  we propose our main algorithms, \emph{Q-Learning with
    UCB2-\{Hoeffding, Bernstein\} exploration}. We prove that these
  two algorithms achieve $\wt{O}(\sqrt{H^{\set{4,3}}SAT})$ regret
  (respectively) and $O(H^3SA\log(K/A))$ local switching cost
  (Section~\ref{section:ql-ucb2}). The
  regret matches their vanilla counterparts of~\cite{jin2018q} but the
  switching cost is only logarithmic in the number of episodes.
\item We show how our low switching cost algorithms can be applied in
  the \emph{concurrent RL} setting~\citep{guo2015concurrent}, in which
  multiple agents can act in parallel
  (Section~\ref{section:concurrent}). The parallelized versions of
  our 
  algorithms with UCB2 exploration give rise to \emph{Concurrent
    Q-Learning} algorithms, which achieve a nearly linear speedup in
  execution time and compares favorably against existing concurrent
  algorithms in sample complexity for exploration.
\item We show a simple $\Omega(HSA)$ lower bound on the switching cost
  for any sublinear regret algorithm, which has at most a
  $O(H^2\log(K/A))$ gap from the upper bound
  (Section~\ref{section:lower-bound}).
\end{compactenum}

\subsection{Prior work}

\paragraph{Low-regret RL} Sample-efficient RL has been studied
extensively since the classical work of~\citet{kearns2002near}
and~\citet{brafman2002r}, with a focus on obtaining a near-optimal
policy in polynomial time, i.e. PAC guarantees. A subsequent line of
work initiate the study of regret in RL and provide algorithms that
achieve regret $\wt{O}(\sqrt{{\rm poly}(H,S,A)\cdot T})$
~\cite{jaksch2010near,osband2013more,agrawal2017optimistic}.  In our
episodic MDP setting, the information-theoretic lower bound for the
regret is $\Omega(\sqrt{H^2SAT})$, which is matched in recent work by
the UCBVI~\cite{azar2017minimax} and ORLC~\cite{dann2018policy}
algorithms. On the other hand, while all the above low-regret
algorithms are essentially model-based, the recent work
of~\cite{jin2018q} shows that model-free algorithms such as Q-learning
are able to achieve $\wt{O}(\sqrt{H^{\set{4,3}}SAT})$ regret which is
only $O(\sqrt{H})$ worse than the lower bound.



\paragraph{Low switching cost / batched algorithms}
\citet{auer2002finite} propose UCB2 in bandit problems, which achieves
the same regret bound as UCB but has switching cost only $O(\log T)$
instead of the naive $O(T)$. \citet{cesa2013online}~study the
switching cost in online learning in both the adversarial and
stochastic setting, and design an algorithm for stochastic bandits
that acheive optimal regert and $O(\log\log T)$ switching cost.

Learning algorithms with switching cost bounded by a fixed $O(1)$
constant is often referred to as \emph{batched algorithms}. Minimax
rates for batched algorithms have been established in various problems
such as bandits~\cite{perchet2016batched,gao2019batched} and convex
optimization~\cite{duchi2018minimax}. In all these scenarios, minimax
optimal $M$-batch algorithms are obtained for all $M$, and their rate
matches that of fully adaptive algorithms once $M=O(\log\log T)$.


\section{Problem setup}
\label{section:problem-setup}
In this paper, we consider undiscounted episodic tabular MDPs of the
form $(H,\mc{S},\P,\mc{A},r)$. The MDP has horizon
$H$ with trajectories of the form
$(x_1,a_1,\dots,x_H,a_H,x_{H+1})$, where $x_h\in\mc{S}$ and $a_h\in\mc{A}$.
The state space $\mc{S}$ and action
space $\mc{A}$ are discrete with $|\mc{S}|=S$ and $|\mc{A}|=A$.
The initial state $x_1$ can be either adversarial (chosen by an
adversary who has access to our algorithm), or stochastic specified by
some distribution $\P_0(x_1)$. For
any $(h,x_h,a_h)\in[H]\times\mc{S}\times\mc{A}$, the transition
probability is denoted as $\P_h(x_{h+1}|x_h,a_h)$. The reward is
denoted as $r_h(x_h,a_h)\in[0,1]$, which we assume to be
deterministic\footnote{Our results can be straightforwardly extended
  to the case with stochastic rewards.}. We assume in addition that
$r_{h+1}(x)=0$ for all $x$, so that the last state $x_{H+1}$ is
effectively an (uninformative) absorbing state.

A deterministic policy $\pi$ consists of $H$ sub-policies
$\pi^h(\cdot):\mc{S}\to \mc{A}$. 
\tx{$\pi_h$?}
For any deterministic policy $\pi$, let $V^\pi_h(\cdot):\mc{S}\to\R$
and $Q^\pi_h(\cdot,\cdot):\mc{S}\times\mc{A}\to\R$ denote its value
function and state-action value function at the $h$-th step
respectively. Let $\pi_\star$ denote an optimal policy, and
$V_h^\star=V_h^{\pi_\star}$ and $Q_h^\star=Q_h^{\pi_\star}$ denote the
optimal $V$ and $Q$ functions for all $h$. As a convenient short hand,
we denote
$[\P_h V_{h+1}](x,a)\defeq \E_{x'\sim \P(\cdot | x,a)}[ V_{h+1}(x')]$
and also use $[\what{\P}_h V_{h+1}](x_h,a_h)\defeq V_{h+1}(x_{h+1})$
in the proofs to denote observed transition. Unless
otherwise specified, we will focus on deterministic policies in this
paper, which will be without loss of generality as there exists at
least one deterministic policy $\pi_\star$ that is optimal.

\paragraph{Regret}
We focus on the regret for measuring the performance of RL
algorithms. Let $K$ be the number of episodes that the agent can play.
(so that total number of steps is $T\defeq KH$.)
The regret of an algorithm is defined as
\begin{equation*}
  \Reg(K) \defeq \sum_{k=1}^K \left[V_1^\star(x_1^k) -
    V_1^{\pi_k}(x_1^k)\right],
\end{equation*}
where $\pi_k$ is the policy it employs before episode $k$
starts, and $V_1^\star$ is the optimal value function for the entire
episode.

\paragraph{Miscellanous notation} We use standard Big-Oh notations in
this paper: $A_n=O(B_n)$ means that there exists
an absolute constant $C>0$ such that $A_n\le CB_n$ (similarly
$A_n=\Omega(B_n)$ for $A_n\ge CB_n$). $A_n=\wt{O}(B_n)$
means that $A_n\le C_nB_n$ where $C_n$ depends at most
poly-logarithmically on all the problem parameters.

\subsection{Measuring adaptivity through local switching cost}
\label{section:switching-cost}
To quantify the adaptivity of RL algorithms, we consider the
following notion of \emph{local switching cost} for RL algorithms.
\begin{definition}
  The local switching cost (henceforth also ``switching cost'')
  between any pair of policies $(\pi,\pi')$ is defined as the number
  of $(h,x)$ pairs on which $\pi$ and $\pi'$ are different:
  \begin{equation*}
    n_\sw(\pi, \pi') \defeq \left| \set{(h,x)\in[H]\times \mc{S}:\pi^h(x)\neq
        [\pi']^h(x)} \right|.
  \end{equation*}
  For an RL algorithm that employs policies
  $(\pi_1,\dots,\pi_K)$, its local switching cost is defined as
  \begin{equation*}
    N_\sw \defeq \sum_{k=1}^{K-1} n_\sw(\pi_k, \pi_{k+1}).
  \end{equation*}
\end{definition}
Note that (1) $N_\sw$ is a random variable in general, as $\pi_k$ can
depend on the outcome of the MDP; (2) we have the trivial bound
$n_\sw(\pi,\pi')\le HS$ for any $(\pi,\pi')$ and
$N_\sw(\mc{A})\le HS(K-1)$ for any algorithm $\mc{A}$.\footnote{To
  avoid confusion, we also note that our local switching cost is not
  to measure the change of the sub-policy $\pi^h$ between timestep $h$
  and $h+1$ (which is in any case needed due to potential
  non-stationarity), but rather to measure the change of the entire
  policy $\pi_k=\set{\pi_k^h}$ between episode $k$ and $k+1$.}



{\bf Remark} The local switching cost extends naturally the notion of
switching cost in online learning~\cite{cesa2013online} and is
suitable in scenarios where the cost of deploying a new policy scales
with the portion of $(h,x)$ on which the action $\pi^h(x)$ is
changed.  A closely related notion of adaptivity is the
\emph{global switching cost}, which simply measures how many times the
algorithm switches its entire policy:
\begin{equation*}
  N_\sw^{\rm gl} = \sum_{k=1}^{K-1} \indic{\pi_k\neq \pi_{k+1}}.
\end{equation*}
As $\pi_k\neq \pi_{k+1}$ implies $n_\sw(\pi_k, \pi_{k+1})\ge 1$, we
have the trivial bound that $N_\sw^{\rm gl}\le N_\sw$. However, the
global switching cost can be substantially smaller for algorithms
that tend to change the policy ``entirely'' rather than
``locally''. In this paper, we focus on bounding $N_\sw$, and leave
the task of tighter bounds on $N_\sw^{\rm gl}$ as future work.

\section{UCB2 for multi-armed bandits}
\label{section:ucb2-bandits}
To gain intuition about the switching cost, we briefly review the UCB2
algorithm~\cite{auer2002finite} on multi-armed bandit problems, which
achieves the same regret bound as the original UCB but has a
substantially lower switching cost.

The multi-armed bandit problem can be viewed as an RL problem with
$H=1$, $S=1$, so that the agent needs only play one action
$a\in\mc{A}$ and observe the (random) reward $r(a)\in[0,1]$. The
distribution of $r(a)$'s are unknown to the agent, and the goal is to
achieve low regret.

The UCB2 algorithm is a variant of the celebrated UCB (Upper
Confidence Bound) algorithm for bandits. UCB2 also maintains upper
confidence bounds on the true means $\mu_1,\dots,\mu_A$, but instead
plays each arm multiple times rather than just once when it's found to
maximize the upper confidence bound. Specifically, when an arm is
found to maximize the UCB for the $r$-th time, UCB2 will play it
$\tau(r) - \tau(r-1)$ times, where
\begin{equation}
  \label{equation:def-tau}
  \tau(r) = (1+\eta)^r
\end{equation}
for $r=0,1,2,\dots$ and some parameter $\eta\in(0,1)$ to be
determined.
\footnote{For convenience, here we treat $(1+\eta)^r$ as an
  integer. In Q-learning we could not make this approximation (as we
  choose $\eta$ super small), and will massage the sequence
  $\tau(r)$ to deal with it.}
The full UCB2 algorithm is presented in Algorithm~\ref{algorithm:ucb2}.

\begin{algorithm}[ht]
   \caption{UCB2 for multi-armed bandits}
   \label{algorithm:ucb2}
\begin{algorithmic}
   \INPUT Parameter $\eta\in(0,1)$.
   \STATE {\bfseries Initialize:} $r_j=0$ for $j=1,\dots,A$. Play each
   arm once. Set $t\setto 0$ and $T\setto T-A$.
   \WHILE{$t\le T$}
   \STATE Select arm $j$ that maximizes $\overline{r}_j+a_{r_j}$,
   where $\overline{r}_j$ is the average reward obtained from arm $j$
   and $a_r= O(\sqrt{\log T/\tau(r)})$ (with some specific choice.)
   \STATE Play arm $j$ exactly $\tau(r_j+1) - \tau(r_j)$ times.
   \STATE Set $t\setto t + \tau(r_j+1) - \tau(r_j)$ and $r_j\setto r_j + 1$.
   \ENDWHILE
\end{algorithmic}
\end{algorithm}

\begin{theorem}[\citet{auer2002finite}]
  \label{theorem:ucb2}
  For $T\ge \max_{i:\mu_i< \mu^\star}\frac{1}{2\Delta_i^2}$, the
  UCB2 algorithm acheives expected regret bound
  \begin{equation*}
    \E\left[ \sum_{t=1}^T (\mu^\star - \mu_t) \right] \le O_\eta\left( \log
      T\cdot \sum_{i:\mu_i<\mu^\star} \frac{1}{\Delta_i} \right),
  \end{equation*}
  where $\Delta_i\defeq \mu^\star - \mu_i$ is the gap between arm $i$
  and the optimal arm. Further, the switching cost is at most
  $O(\frac{A\log (T/A)}{\eta})$.
\end{theorem}
The switching cost bound in Theorem~\ref{theorem:ucb2} comes directly
from the fact that $\sum_{i=1}^A(1+\eta)^{r_i}\le T$ implies
$\sum_{i=1}^A r_i\le O(A\log(T/A)/\eta)$, by the convexity of
$r\mapsto (1+\eta)^r$ and Jensen's inequality. Such an approach can be
fairly general, and we will follow it in sequel to develop RL
algorithm with low switching cost.


\section{Q-learning with UCB2 exploration}
\label{section:ql-ucb2}
In this section, we propose our main algorithm, Q-learning with UCB2
exploration, and show that it achieves sublinear regret as well as
logarithmic local switching cost.

\subsection{Algorithm description}
\label{section:ql-ucb2-algorithm}
\paragraph{High-level idea}
Our algorithm maintains wo sets of optimistic $Q$ estimates: a
\emph{running estimate} $\wt{Q}$ which is updated after every episode,
and a \emph{delayed estimate} $Q$ which is only updated occasionally
but used to select the action. In between two updates to $Q$, the
policy stays fixed, so the number of policy switches is bounded by the
number of updates to $Q$.

To describe our algorithm, let $\tau(r)$ be defined as
\begin{align*}
  \tau(r) = \ceil{(1+\eta)^r},~~~r=1,2,\dots
\end{align*}
and define the \emph{triggering sequence} as
\begin{equation}
  \label{equation:switching-sequence}
  \begin{aligned}
    & \set{t_n}_{n\ge 1} = \set{1, 2, \dots, \tau(r_\star)} \cup \set{
      \tau(r_\star+1), \tau(r_\star+2),\dots},
  \end{aligned}
\end{equation}
where the parameters $(\eta, r_\star)$ will be inputs to the
algorithm.
Define for all $t\in\set{1,2,\dots}$ the quantities
\begin{align*}
  \tau_{\last}(t) \defeq \max\set{t_n:t_n\le t}~~~{\rm
  and}~~~\alpha_t = \frac{H+1}{H+t}. 
\end{align*}

\paragraph{Two-stage switching strategy}
The triggering sequence~\eqref{equation:switching-sequence} defines a
\emph{two-stage strategy} for switching policies. Suppose for a given
$(h, x_h)$, the algorithm decides to take some particular $a_h$ for
the $t$-th time, and has observed $(r_h, x_{h+1})$ and updated the
running estimate $\wt{Q}_h(x_h,a_h)$ accordingly. Then, whether to
also update the policy network $Q$ is decided as
\begin{itemize}
\item Stage I: if $t\le \tau(r_\star)$, then always perform the update
  $Q_h(x_h,a_h) \setto \wt{Q}_h(x_h,a_h)$.
\item Stage II: if $t>\tau(r_\star)$, then perform the above update
  only if $t$ is in the triggering sequence,
  that is, $t=\tau(r)=\ceil{(1+\eta)^r}$ for some $r>r_\star$.
\end{itemize}
In other words, for any state-action pair, the algorithm performs
eager policy update in the beginning $\tau(r_\star)$ visitations, and
switches to delayed policy update after that according to UCB2
scheduling.

\paragraph{Optimistic exploration bonus} We employ either a
Hoeffding-type or a Bernstein-type exploration bonus to make sure that
our running $Q$ estimates are optimistic. The full algorithm with
Hoeffding-style bonus is presented in
Algorithm~\ref{algorithm:ql-ucb2}.

\begin{algorithm*}[ht]
   \caption{Q-learning with UCB2-Hoeffding (UCB2H) Exploration}
   \label{algorithm:ql-ucb2}
\begin{algorithmic}
   \INPUT Parameter $\eta\in(0,1)$, $r_\star\in\Z_{>0}$, and $c>0$.
   \STATE {\bfseries Initialize:} $\wt{Q}_h(x,a)\setto H$,
   $Q_h\setto \wt{Q}_h$, $N_h(x,a)\setto 0$ for all
   $(x,a,h)\in\mc{S}\times \mc{A}\times [H]$.
   \FOR{episode $k=1,\dots,K$}
   \STATE Receive $x_1$.
   \FOR{step $h=1,\dots, H$}
   \STATE Take action $a_h\setto \argmax_{a'}Q_h(x_h,a')$, and observe
   $x_{h+1}$.
   \STATE $t=N_h(x_h, a_h)\setto N_h(x_h, a_h) + 1$;
   \STATE $b_t = c\sqrt{H^3\ell/t}$ (Hoeffding-type bonus);
   \STATE $\wt{Q}_h(x_h, a_h)\setto (1-\alpha_t)\wt{Q}_h(x_h, a_h) +
   \alpha_t[r_h(x_h, a_h) + \wt{V}_{h+1}(x_{h+1}) + b_t]$.
   \STATE $\wt{V}_h(x_h)\setto \min\set{H, \max_{a'\in\mc{A}}\wt{Q}_h(x_h,
     a')}$.
   {\color{red!80!black}
   \IF{$t \in \set{t_n}_{n\ge 1}$ (where $t_n$ is defined
     in~\eqref{equation:switching-sequence})}
   \STATE (Update policy) $Q_h(x_h,\cdot) \setto \wt{Q}_h(x_h,\cdot)$.
   \ENDIF}
   \ENDFOR
   \ENDFOR
\end{algorithmic}
\end{algorithm*}

\subsection{Regret and switching cost guarantee}
\label{section:ql-ucb2-result}
We now present our main results.
\begin{theorem}[Q-learning with UCB2H exploration achieves sublinear
  regret and low switching cost]
  \label{theorem:ql-ucb2}
  Choosing $\eta=\frac{1}{2H(H+1)}$ and
  $r_\star=\ceil{\frac{\log(10H^2)}{\log(1+\eta)}}$,
  with probability at least $1-p$, the regret of
  Algorithm~\ref{algorithm:ql-ucb2} is bounded by
  $\wt{O}(\sqrt{H^4SAT})$.
  Further, the local switching cost is bounded as
  $N_\sw\le O(H^3SA\log(K/A))$.
\end{theorem}
Theorem~\ref{theorem:ql-ucb2} shows that the total regret of
Q-learning with UCB2 exploration is $\wt{O}(\sqrt{H^4SAT})$, the same
as UCB version of~\cite{jin2018q}. In addition, the local switching
cost of our algorithm is only $O(H^3SA\log (K/A))$, which is
logarithmic in $K$, whereas the UCB version can have in the worst case
the trivial bound $HS(K-1)$. We give a high-level overview of the proof
Theorem~\ref{theorem:ql-ucb2} in Section~\ref{section:proof-highlight},
and defer the full proof to Appendix~\ref{section:proof-ql-ucb2}.

\paragraph{Bernstein version} Replacing the Hoeffding bonus with a
Bernstein-type bonus, we can achieve $\wt{O}(\sqrt{H^3SAT})$ regret
($\sqrt{H}$ better than UCB2H) and the same switching cost bound.
\begin{theorem}[Q-learning with UCB2B exploration achieves sublinear
  regret and low switching cost]
  \label{theorem:ql-ucb2-Bernstein}
  Choosing $\eta=\frac{1}{2H(H+1)}$ and
  $r_\star=\ceil{\frac{\log(10H^2)}{\log(1+\eta)}}$, with probability
  at least $1-p$, the regret of
  Algorithm~\ref{algorithm:ql-ucb2-Bernstein} is bounded by
  $\wt{O}(\sqrt{H^3SAT})$ as long as
  $T=\wt{\Omega}(H^6S^2A^2)$. Further, the local switching cost is
  bounded as $N_\sw\le O(H^3SA\log(K/A))$.
\end{theorem}
The full algorithm description, as well as the proof of
Theorem~\ref{theorem:ql-ucb2-Bernstein}, are deferred to
Appendix~\ref{appendix:proof-ql-bernstein}.

Compared with Q-learning with UCB~\cite{jin2018q},
Theorem~\ref{theorem:ql-ucb2} and~\ref{theorem:ql-ucb2-Bernstein}
demonstrate that ``vanilla'' low-regret RL algorithms such as
Q-Learning can be turned into low switching cost versions without any
sacrifice on the regret bound. 

\subsection{PAC guarantee}
Our low switching cost algorithms can also achieve the PAC
learnability guarantee. Specifically, we have the following
\begin{corollary}[PAC bound for Q-Learning with UCB2 exploration]
  \label{corollary:pac-bound}
  Suppose (WLOG) that $x_1$ is deterministic. For any $\eps>0$,
  Q-Learning with \{UCB2H, UCB2B\} exploration can output a
  (stochastic) policy $\what{\pi}$ such that with high probability
  \begin{equation*}
    V_1^\star(x_1) - V_1^{\what{\pi}}(x_1) \le \eps
  \end{equation*}
  after $K=\wt{O}(H^{\set{5,4}}SA/\eps^2)$ episodes.
\end{corollary}
The proof of Corollary~\ref{corollary:pac-bound} involves turning the
regret bounds in Theorem~\ref{theorem:ql-ucb2}
and~\ref{theorem:ql-ucb2-Bernstein} to PAC bounds using the
online-to-batch conversion, similar as in~\cite{jin2018q}.  The full
proof is deferred to Appendix~\ref{appendix:proof-pac-bound}.

\section{Application: Concurrent Q-Learning}
\label{section:concurrent}
Our low switching cost Q-Learning can be applied to developing
algorithms for \emph{Concurrent RL}~\cite{guo2015concurrent} -- a
setting in which multiple RL agents can act in parallel and
hopefully accelerate the exploration in wall time.

\paragraph{Setting} We assume there are $M$ agents / machines, where
each machine can interact with a independent copy of the episodic MDP
(so that the transitions and rewards on the $M$ MDPs are mutually
independent). Within each episode, the $M$ machines must play
synchronously and cannot communiate, and can only exchange information
after the entire episode has finished. Note that our setting is in a
way more stringent than~\cite{guo2015concurrent}, which allows
communication after each timestep.

We define a ``round'' as the duration in which the $M$ machines
simultanesouly finish one episode and (optionally) communicate and
update their policies. We measure the performance of a concurrent
algorithm in its required number of rounds to find an $\eps$
near-optimal policy. With larger $M$, we expect such number of rounds
to be smaller, and the best we can hope for is a \emph{linear speedup}
in which the number of rounds scales as $M^{-1}$.


\paragraph{Concurrent Q-Learning}
Intuitively, any low switching cost algorithm can be made into a
concurrent algorithm, as its execution can be parallelized in between
two consecutive policy switches. Indeed, we can design concurrent
versions of our low switching Q-Learning algorithm and achieve a
nearly linear speedup.
\begin{theorem}[Concurrent Q-Learning achieves nearly linear speedup]
  \label{theorem:concurrent}
  There exists concurrent versions of Q-Learning with
  \{UCB2H, UCB2B\} exploration such that, given a budget of $M$
  parallel machines, returns an $\eps$ near-optimal policy in
  \begin{equation*}
    \wt{O}\left(H^3SA + \frac{H^{\set{5,4}}SA}{\eps^2M} \right)
  \end{equation*}
  rounds of execution.
\end{theorem}
Theorem~\ref{theorem:concurrent} shows that concurrent Q-Learning has
a linear speedup so long as $M=\wt{O}(H^{\set{2,1}}/\eps^2)$. In
particular, in high-accuracy (small $\eps$) cases, the constant
overhead term $H^3SA$ can be negligible and we essentially have a
linear speedup over a wide range of $M$. The proof of
Theorem~\ref{theorem:concurrent} is deferred to
Appendix~\ref{appendix:concurrent}.

\paragraph{Comparison with existing concurrent algorithms}
Theorem~\ref{theorem:concurrent} implies a PAC mistake bound as well:
there exists concurrent algorithms on $M$ machines, Concurrent
Q-Learning with \{UCB2H, UCB2B\}, that performs a $\eps$ near-optimal
action on all but
\begin{equation*}
  \wt{O}\left( H^4SAM + \frac{H^{\set{6,5}}SA}{\eps^2} \right) \defeq
  N_\eps^{\sf CQL}
\end{equation*}
actions with high probability (detailed argument in
Appendix~\ref{appendix:concurrent-mistake-bound}).

We compare ourself with the Concurrent MBIE (CMBIE) algorithm
in~\cite{guo2015concurrent}, which considers the discounted and
infinite-horizon MDPs, and has a mistake
bound\footnote{$(S',A',\gamma')$ are the \{\# states, \# actions,
  discount factor\} of the discounted infinite-horizon MDP.}
\begin{equation*}
  \wt{O}\left( \frac{S'A'M}{\eps(1-\gamma')^2} +
    \frac{S'^2A'}{\eps^3(1-\gamma')^6} \right) \defeq N_\eps^{\sf CMBIE}
\end{equation*}
Our concurrent Q-Learning compares favorably against CMBIE in terms of
the mistake bound:
\begin{itemize}
\item Dependence on $\eps$. CMBIE achieves
  $N_\eps^{\sf CMBIE}= \wt{O}(\eps^{-3}+\eps^{-1}M)$, whereas our
  algorithm achieves $N_\eps^{\sf CQL}=\wt{O}(\eps^{-2}+M)$,
  better by a factor of $\eps^{-1}$.
\item Dependence on $(H,S,A)$. These are not comparable in general,
  but under the ``typical'' correspondence\footnote{One can transform
  an episodic MDP with $S$ states to an infinite-horizon MDP with $HS$
  states. Also note that the ``effective'' horizon for discounted
  MDP is $(1-\gamma)^{-1}$.} $S'\setto HS$,
  $A'\setto A$, $(1-\gamma')^{-1}\setto H$, we get
  $N_\eps^{\sf CMBIE}=\wt{O}(H^3SAM\eps^{-1} +
  H^8S^2A\eps^{-3})$. Compared to $N_\eps^{\sf CQL}$, CMBIE has a higher
  dependence on $H$ as well as a $S^2$ term due to its
  model-based nature.
\end{itemize}
\section{Proof overview of Theorem~\ref{theorem:ql-ucb2}}
\label{section:proof-highlight}
The proof of Theorem~\ref{theorem:ql-ucb2} involves two parts: the
switching cost bound and the regret bound. The switching cost bound
results directly from the UCB2 switching schedule, similar as
in the bandit case (cf. Section~\ref{section:ucb2-bandits}). However,
such a switching schedule results in delayed policy updates, which
makes establishing the regret bound technically challenging.

The key to the $\wt{O}({\rm poly}(H)\cdot \sqrt{SAT})$ regret bound
for ``vanilla'' Q-Learning in~\cite{jin2018q} is a \emph{propagation
  of error} argument, which shows that the regret\footnote{Technically
  it is an upper bound on the regret.} from the $h$-th step and
forward (henceforth the $h$-regret), defined as
\begin{equation*}
  \sum_{k=1}^K \wt{\delta}_h^k \defeq \sum_{k=1}^K \left[\wt{V}_h^k
  -V_h^{\pi_k}\right](x_h^k), 
\end{equation*}
is bounded by $1+1/H$ times
the $(h+1)$-regret, plus some bounded error term. As
$(1+1/H)^H=O(1)$, this fact can be applied recursively for
$h=H,\dots,1$ which will result in a total regret bound that is not
exponential in $H$. The control of the (excess) error propagation
factor by $1/H$ and the ability to converge are then achieved
simultaneously via the stepsize choice $\alpha_t=\frac{H+1}{H+t}$.

In constrast, our low-switching version of Q-Learning updates the
exploration policy in a delayed fashion according to the UCB2
schedule. Specifically, the policy at episode $k$ does not correspond
to the argmax of the running estimate $\wt{Q}^k$, but rather a
previous version $Q^k=\wt{Q}^{k'}$ for some $k'\le k$. This introduces
a mismatch between the $Q$ used for exploration and the $Q$ being
updated, and it is a priori possible whether such a mismatch will blow
up the propagation of error.

We resolve this issue via a novel error analysis, which at a
high level consists of the following steps:
\begin{enumerate}[(i)]
\item We show that the quantity $\wt{\delta}_h^k$ is upper bounded by
  a \emph{max error}
  \begin{equation*}
    \wt{\delta}_h^k \le \left(\max\set{\wt{Q}_h^{k'}, \wt{Q}_h^k} -
      Q_h^{\pi_k}\right)(x_h^k, a_h^k) = \left(\wt{Q}_h^{k'} -
      Q_h^{\pi_k} + \left[\wt{Q}_h^{k} - 
        \wt{Q}_h^{k'}\right]_+ \right)(x_h^k, a_h^k)
  \end{equation*}
  (Lemma~\ref{lemma:max-error}). On the right hand side, the first
  term $\wt{Q}_h^{k'}-Q_h^{\pi_k}$ does not have a mismatch (as
  $\pi_k$ depends on $\wt{Q}^{k'}$) and can be bounded similarly as
  in~\cite{jin2018q}. The second term $[\wt{Q}_h^k -\wt{Q}_h^{k'}]_+$
  is a perturbation term, which we bound in a precise way that relates
  to stepsizes in between episodes $k'$ to $k$ and the $(h+1)$-regret
  (Lemma~\ref{lemma:perturbation-bound}).
\item We show that, under the UCB2 scheduling, the combined error
  above results a mild blowup in the relation between $h$-regret and
  $(h+1)$-regret -- the multiplicative factor can be now bounded by
  $(1+1/H)(1+O(\eta H))$
  (Lemma~\ref{lemma:double-counting}). Choosing $\eta=O(1/H^2)$ will
  make the multiplicative factor $1+O(1/H)$ and the propagation of
  error argument go through.
\end{enumerate}
We hope that the above analysis can be applied more broadly in
analyzing exploration problems with delayed updates or asynchronous
parallelization. 


\section{Lower bound on switching cost}
\label{section:lower-bound}


\begin{theorem}
  \label{theorem:lower-bound}
  Let $A\ge 4$ and $\mc{M}$ be the set of episodic MDPs satisfying the
  conditions in Section~\ref{section:problem-setup}. For any RL
  algorithm $\mc{A}$ satisfying $N_\sw\le HSA/2$, we have
  \begin{equation*}
    \sup_{M\in\mc{M}} \E_{x_1, M}\left[\sum_{k=1}^K
      V_1^\star(x_1) - V_1^{\pi_k}(x_1)\right] \ge KH/4.
  \end{equation*}
  i.e. the worst case regret is linear in $K$.
\end{theorem}
Theorem~\ref{theorem:lower-bound} implies that the switching cost of
any no-regret algorithm is lower bounded by $\Omega(HSA)$, which is
quite intuitive as one would like to play each action at least once on
all $(h,x)$. Compared with the lower bound, the switching cost
$O(H^3SA\log K)$ we achieve through UCB2 scheduling is at most off by
a factor of $O(H^2\log K)$. We believe that the $\log K$ factor is not
necessary as there exist algorithms achieving
double-log~\cite{cesa2013online} in bandits, and would also like to
leave the tightening of the $H^2$ factor as future work. The proof of
Theorem~\ref{theorem:lower-bound} is deferred to
Appendix~\ref{appendix:proof-lower-bound}.



\section{Conclusion}
In this paper, we take steps toward studying limited adaptivity RL. We
propose a notion of local switching cost to account for the adaptivity
of RL algorithms. We design a Q-Learning algorithm with infrequent
policy switching that achieves $\wt{O}(\sqrt{H^\set{4,3}SAT})$ regret
while switching its policy for at most $O(\log T)$ times. Our
algorithm works in the concurrent setting through parallelization and
achieves nearly linear speedup and favorable sample complexity. Our
proof involves a novel perturbation analysis for exploration
algorithms with delayed updates, which could be of broader interest.

There are many interesting future directions, including (1) low
switching cost algorithms with tighter regret bounds, most likely via
model-based approaches; (2) algorithms with even lower switching cost;
(3) investigate the connection to other settings such as off-policy RL.

\section*{Acknowledgment}
The authors would like to thank Emma Brunskill, Ramtin Keramati,
Andrea Zanette, and the staff of CS234 at Stanford for the valuable
feedback at an earlier version of this work, and Chao Tao for the very
insightful feedback and discussions on the concurrent Q-learning
algorithm. YW was supported by a start-up grant from UCSB CS
department, NSF-OAC 1934641 and a gift from AWS ML Research Award.

\bibliographystyle{abbrvnat}
\bibliography{bib}

\appendix


\section{Proof of Theorem~\ref{theorem:ql-ucb2}}
\label{section:proof-ql-ucb2}
This section is structured as follows. We collect notation
in Section~\ref{section:proof-notation} and list some
basic properties of the running estimate $\wt{Q}$ in
Section~\ref{section:proof-basics}, establish useful perturbation
bounds on $[\wt{Q}_h^k - \wt{Q}_h^{k'}]_+$ in
Section~\ref{section:proof-propagation-error}, and present the proof of
the main theorem
in Section~\ref{section:proof-main}.

\subsection{Notation}
\label{section:proof-notation}
Let $\wt{Q}_h^k(x,a)$ and $Q_h^k(x,a)$ denote the estimates $\wt{Q}$ and
$Q$ in Algorithm~\ref{algorithm:ql-ucb2} before the $k$-th episode has
started. Note that $\wt{Q}_h^1(x,a)=Q_h^1(x,a)\equiv H$.

Define the sequences
\begin{align*}
  \alpha_t^0 \defeq \prod_{i=1}^t(1-\alpha_i),~~\alpha_t^i \defeq
  \alpha_i\cdot \prod_{\tau=i+1}^t (1-\alpha_\tau).
\end{align*}
For $t\ge 1$, we have $\alpha_t^0=0$ and
$\sum_{i=1}^t\alpha_t^i=1$. For $t=0$, we have $\alpha_t^0=1$.

With the definition of $\alpha_t^i$ in hand,
we have the following explicit formula for $\wt{Q}_h^k$:
\begin{align*}
  \wt{Q}_h^k(x,a) = \alpha_t^0 H + \sum_{i=1}^t \alpha_t^i\left(
  r_h(x,a) + \wt{V}^{k_i}_{h+1}(x_{h+1}^{k_i}) + b_i \right),
\end{align*}
where $t$ is the number of updates on $\wt{Q}_h(x,a)$ {\bf prior to}
the $k$-th epoch, and $k_1,\dots,k_t$ are the indices for the
epochs. Note that $k=k_{t+1}$ if the algorithm indeed
observes $x$ and takes the action $a$ on the $h$-th step of episode
$k$. 

Throughout the proof we let $\ell\defeq \log(SAT/p)$ denote a log
factor, where we recall $p$ is the pre-specified tail probability. 

\yub{More notation to be collected. $\ell$ is a log, etc.}

\subsection{Basics}
\label{section:proof-basics}
\begin{lemma}[Properties of $\alpha_t^i$; Lemma 4.1,~\cite{jin2018q}] 
  \label{lemma:alphati}
  The following properties hold for the sequence $\alpha_t^i$:
  \begin{enumerate}[(a)]
  \item $\frac{1}{\sqrt{t}} \le \sum_{i=1}^t
    \frac{\alpha_t^i}{\sqrt{i}} \le \frac{2}{\sqrt{t}}$ for every
    $t\ge 1$.
  \item $\max_{i\in[t]}\alpha_t^i\le \frac{2H}{t}$ and $\sum_{i=1}^t
    (\alpha_t^i)^2 \le \frac{2H}{t}$ for every $t\ge 1$.
  \item $\sum_{t=i}^\infty \alpha_t^i = 1 + \frac{1}{H}$ for every
    $i\ge 1$.
  \end{enumerate}
\end{lemma}
\begin{lemma}[$\wt{Q}$ is optimistic and accurate; Lemma 4.2 \&
  4.3,~\cite{jin2018q}]
  \label{lemma:optimistic-and-accurate}
  We have for all $(h,x,a,k)\in[H]\times\mc{S}\times\mc{A}\times[K]$
  that
  \begin{equation}
    \label{equation:q-diff}
  \begin{aligned}
    & \quad \wt{Q}_h^k(x,a) - Q_h^\star(x,a) \\
    & = \alpha_t^0(H - Q_h^\star(x,a)) + 
    \sum_{i=1}^t \alpha_t^i\bigg( r_h(x,a) +
    \wt{V}_{h+1}^{k_i}(x_{h+1}^{k_i}) - V^\star_{h+1}(x_{h+1}^{k_i})
    + \left[\left( \what{\P}_h^{k_i} - \P_h \right)V_{h+1}^\star\right](x,a) +
    b_i\bigg),
    \end{aligned}
  \end{equation}
  where $[\hat{\P}_h^{k_i} V_{h + 1}](x,a) \defeq V_{h + 1}(x_{h+1}^{k_i})$.
  
  Further, with probability at least $1-p$, choosing
  $b_t=c\sqrt{H^3\ell/t}$
  for some absolute constant
  $c>0$, we have for all $(h,x,a,k)$ that
  \begin{align*}
    & 0\le \wt{Q}_h^k(x,a) - Q_h^\star(x,a)
      \le \alpha_t^0H +
      \sum_{i=1}^t\alpha_t^i(\wt{V}_{h+1}^{k_i} -
      V_{h+1}^\star)(x_{h+1}^{k_i}) + \beta_t
  \end{align*}
  where $\beta_t\defeq 2\sum_{i=1}^t\alpha_t^ib_i\le
  4c\sqrt{H^3\ell/t}$.
\end{lemma}
{\bf Remark.} This first part of the Lemma, i.e. the expression of
$\wt{Q}_h^k-Q_h^\star$ in terms of rewards and value functions,
is an aggregated form for the $Q$ functions under the Q-Learning
updates, and is independent to the actual exploration policy as well
as the bonus.



\subsection{Perturbation bound under delayed Q updates}
\label{section:proof-propagation-error}
For any $(h,k)\in[H]\times[K]$, let
\begin{align}
\label{eq:def_delta_phi}
  \wt{\delta}_h^k \defeq \left(\wt{V}_h^k -
  V^{\pi_k}_h\right)(x_h^k),~~~\wt{\phi}_h^k \defeq \left(\wt{V}_h^k
  - V_h^\star\right)(x_h^k)
\end{align}
denote the errors of the estimated $\wt{V}_h^k$ relative to $V^{\pi_k}$
and $V^\star$. As $\wt{Q}$ is optimistic, the regret can be bounded
as
\begin{align*}
  & \Reg(K) = \sum_{k=1}^K \left[ V_1^\star(x_1^k) - V_1^{\pi_k}(x_1^k)
    \right] \le \sum_{k=1}^K \left[ \wt{V}_1^k(x_1^k) -
    V_1^{\pi_k}(x_1^k) \right] = \sum_{k=1}^K \wt{\delta}_1^k.
\end{align*}
The goal of the propagation of error is to related
$\sum_{k=1}^K\wt{\delta}_h^k$ by $\sum_{k=1}^K\wt{\delta}_{h+1}^k$.

We begin by showing that $\wt{\delta}_h^k$ is controlled by the max of
$\wt{Q}_h^k$ and $\wt{Q}_h^{k'}$, where $k'=k_{\tau_\last(t)+1}$.
\begin{lemma}[Max error under delayed policy update]
  \label{lemma:max-error}
  We have
  \begin{equation}
    \label{equation:delayed-behavior}
    \begin{aligned}
      \wt{\delta}_h^k
      \le \left(\max\set{\wt{Q}_h^{k'}, \wt{Q}_h^k} -
        Q_h^{\pi_k}\right)(x_h^k, a_h^k) = \left(\wt{Q}_h^{k'} -
        Q_h^{\pi_k} + \left[\wt{Q}_h^{k} - 
          \wt{Q}_h^{k'}\right]_+ \right)(x_h^k, a_h^k).
    \end{aligned}
  \end{equation}
  where $k'=k_{\tau_\last(t)+1}$ (which depends on $k$.) In particular,
  if $t=\tau_\last(t)$, then $k=k'$ and the upper bound reduces to
  $(\wt{Q}_h^{k'} - Q_h^{\pi_k})(x_h^k, a_h^k)$.
  

\end{lemma}
\begin{proof}
  We first show~\eqref{equation:delayed-behavior}.
  By definition of $\pi_k$ we have $V_h^{\pi_k}(x_h^k) =
  Q_h^{\pi_k}(x_h^k, a_h^k)$,

  so it suffices to show that
  $$\wt{V}_h^k(x_h^k) \le \max\set{\wt{Q}_h^k(x_h^k, a_h^k),
    \wt{Q}_h^{k'}(x_h^k, a_h^k)}.$$
  Indeed, we have
  \begin{equation*}
    \wt{V}_h^k(x_h^k) = \min\set{H, \max_{a'}\wt{Q}_h^k(x_h^k, a')}
    \le \max_{a'}\wt{Q}_h^k(x_h^k, a').
  \end{equation*}
  On the other hand, $a_h^k$ maximizes $Q_h(x_h^k,\cdot)$. Due to the
  scheduling of the delayed update, $Q_h(x_h^k,\cdot)$ was set to
  $\wt{Q}_h^{k_{\tau_\last(t)}+1}(x_h^k,\cdot)$, and
  $\wt{Q}_h^{\wt{k}}(x_h^k,a_h^k)$ was not updated since then before
  $\wt{k}=k'=k_{\tau_\last(t)+1}$, so
  $Q_h(x_h^k,\cdot)=\wt{Q}_h^{k'}(x_h^k, \cdot)$.

  Now, defining
  \begin{equation*}
    q_{\rm old}(\cdot) \defeq \wt{Q}_h^{k'}(x_h^k,\cdot),~~~q_{\rm
      new}(\cdot) \defeq \wt{Q}_h^k(x_h^k,\cdot),
  \end{equation*}
  the vectors $q_{\rm old}$ and $q_{\rm new}$ only
  differ in the $a_h^k$-th component (which is the only action taken
  therefore also the only component that is updated). If $q_{\rm new}$
  is also maximized at $a_h^k$, then we have $\wt{V}_h^k(x_h^k)\le
  q_{\rm new}(a_h^k)$; otherwise it is maximized at some $a'\neq a_h^k$
  and we have
  \begin{equation*}
    \wt{V}_h^k(x_h^k) \le q_{\rm new}(a') = q_{\rm old}(a') \le
    \max_{a} q_{\rm old}(a) = \wt{Q}_h^{k'}(x_h^k, a_h^k).
  \end{equation*}
  Putting together we get
  \begin{equation*}
    \wt{V}_h^k(x_h^k) \le \max\set{\wt{Q}_h^k(x_h^k, a_h^k),
      \wt{Q}_h^{k'}(x_h^k, a_h^k)},
  \end{equation*}
  which implies~\eqref{equation:delayed-behavior}.

  
\end{proof}

Lemma~\ref{lemma:max-error} suggests bounding $\wt{\delta}_h^k$ via
bounding the ``main term'' $\wt{Q}^{k'} - Q_h^{\pi_k}$ and
``perturbation term'' $[\wt{Q}_h^k - \wt{Q}_h^{k'}]_+$ separately. We
now establish the bound on the perturbation term.
\begin{lemma}[Perturbation bound on $(\wt{Q}_h^k - \wt{Q}_h^{k'})_+$]
  \label{lemma:perturbation-bound}
  For any $k$ such that $k>k'$ (so that the perturbation term is
  non-zero), we have
  \begin{equation}
    \label{equation:perturbation-bound}
    \begin{aligned}
      &  \left[\wt{Q}_h^{k} - \wt{Q}_h^{k'}\right]_+(x_h^k, a_h^k) \le
      \beta_t +
      \sum_{i=\tau_\last(t)+1}^t\alpha_t^i\wt{\phi}_{h+1}^{k_i} +
      \overline{\zeta}_h^k,
    \end{aligned}
  \end{equation}
  where
  \begin{align*}
    \overline{\zeta}_h^k \defeq \left|\sum_{i=\tau_\last(t)+1}^t
    \alpha_t^i[(\what{\P}_h^k - \P_h)V_{h+1}^\star](x_h^k,
    a_h^k) \right|
  \end{align*}
  and w.h.p. we have uniformly over all $(h,k)$ that $\overline{\zeta}_h^k\le
  C\sqrt{H^3\ell/t}$
  for some absolute constant $C>0$.
\end{lemma}
\begin{proof}
  Throughout this proof we will omit the arguments $(x_h^k, a_h^k)$ in
  $\wt{Q}_h$ and $r_h$ as they are clear from the context.
  By the update formula for $\wt{Q}$ in
  Algorithm~\ref{algorithm:ql-ucb2}, we get
  \begin{align*}
    \wt{Q}_h^k = \left(\prod_{i=\tau_\last(t)+1}^t(1-\alpha_i)
    \right)\wt{Q}_h^{k'}  + \sum_{i=\tau_\last(t)+1}^t \alpha_t^i
    \left[r_h(x_h^k, a_h^k) + 
    \wt{V}_{h+1}^{k_i}(x_{h+1}^{k_i}) + b_i\right].
  \end{align*}
  Subtracting $\wt{Q}_h^{k'}$ on both sides (and noting that $\left(\prod_{i=\tau_\last(t)+1}^t(1-\alpha_i)
  \right) +  \sum_{i=\tau_\last(t)+1}^t \alpha_t^i = 1$), we get
  \begin{equation}
    \label{equation:perturbation-decomposition}
    \begin{aligned}
      \wt{Q}_h^k - \wt{Q}_h^{k'} = \sum_{i=\tau_\last(t)+1}^t
      \alpha_t^i \underbrace{\left[r_h + 
          \wt{V}_{h+1}^{k_i}(x_{h+1}^{k_i}) + b_i - \wt{Q}_h^{k'}
        \right]}_{d_i}. 
    \end{aligned}
  \end{equation}
  We now upper bound $d_i$ for each $i$. Adding and subtracting
  $Q_h^\star$, we obtain
  \begin{align*}
    d_i
    & = \left( r_h +
      \wt{V}_{h+1}^{k_i}(x_{h+1}^{k_i}) + b_i - Q_h^\star  \right) -
      (\wt{Q}_h^{k'} - Q_h^\star) \\ 
    & \stackrel{\text{(i)}}{=} 
      \wt{V}_{h+1}^{k_i}(x_{h+1}^{k_i}) - V^\star(x_{h+1}^{k_i}) +
      (\what{\P}_h^k - \P_h)V_{h+1}^\star  + b_i -
      \left(\wt{Q}_h^{k'} - Q_h^\star\right) \\
    & \stackrel{\text{(ii)}}{\le} b_i + \wt{\phi}_{h+1}^{k_i} + 
      \underbrace{(
      \what{\P}_h^k - \P_h )V_{h+1}^\star}_{\defeq \zeta_i}.
  \end{align*}
  where (i) follows from the Bellman optimality equation on $Q_h^\star$, and that $[\what{\P}_h^k V_{h+1}^\star](x_h^k, a_h^k) = V_{h+1}^\star(x_{h+1}^k)$
  and (ii) follows from the optimistic property of
  $\wt{Q}^{k'}_h$ (from
  Lemma~\ref{lemma:optimistic-and-accurate}) and the definition of
  $\wt{\phi}_{h+1}^{k_i}$.  
  Substituting this into~\eqref{equation:perturbation-decomposition}
  gives
  \begin{align*}
    \left[\wt{Q}_h^{k} - \wt{Q}_h^{k'} \right]_+ \le
    \left[\sum_{i=\tau_\last(t)+1}^t \alpha_t^i \left(b_i +
    \wt{\phi}_{h+1}^{k_i} + \zeta_i\right) \right]_+ \le \beta_t +
    \sum_{i=\tau_\last(t)+1}^t 
    \alpha_t^i\wt{\phi}_{h+1}^{k_i} + \underbrace{\left|
    \sum_{i=\tau_\last(t)+1}^t \alpha_t^i\zeta_i 
    \right|}_{\overline{\zeta}^k_h}. 
  \end{align*}
  Finally, note that $\zeta_i$ is a martingale difference sequence, so
  we can apply the Azuma-Hoeffding inequality to get that
  \begin{align*}
    & \overline{\zeta}_h^k \le c\sqrt{\sum_{i=\tau_\last(t)+1}^t
      (\alpha_t^i)^2 H^2\ell} \stackrel{\text{(i)}}{\le}
      c\sqrt{\frac{2H}{t}\cdot H^2\ell} = 
      C\sqrt{\frac{H^3\ell}{t}}
  \end{align*}
  uniformly over $(h,k)$, where (i) follows from
  Lemma~\ref{lemma:alphati}(b).
\end{proof}

\subsection{Proof of Theorem~\ref{theorem:ql-ucb2}}
\label{section:proof-main}
Proof of the main theorem is done through combining the
perturbation bound and the ``main term'', and showing that the
propagation of error argument still goes through.

\begin{lemma}[Error accumulation under delayed update]
  \label{lemma:double-counting}
  Suppose we choose $\eta=\frac{1}{2H(H+1)}$ and
  $r_\star=\ceil{\frac{\log(10H^2)}{\log(1+\eta)}}$ for
  the triggering sequence~\eqref{equation:switching-sequence} then we
  have for all $i$ that
  \begin{align*}
      \sum_{t:t\ge i, \tau_\last(t)\le i-1} \alpha_t^i + \sum_{t:\tau_\last(t)\ge i}
    \alpha_{\tau_\last(t)}^i \le 1 + 3/H.
  \end{align*}
\end{lemma}
\begin{proof}
  Let $\wt{S}_i$ denote the above sum. 
  We compare $\wt{S}_i$ with
  \begin{align*}
    S_i \defeq \sum_{t=i}^\infty \alpha_t^i  = 1 + \frac{1}{H},
  \end{align*}
  where the last equality follows from Lemma~\ref{lemma:alphati}(c).

  Let us consider $\wt{S}_i - S_i$ by looking at the difference of the
  individual terms for each $t\ge i$. When taking the difference, the
  term $\sum_{t:t\ge i, \tau_\last(t)\le i-1} \alpha_t^i$ will vanish,
  and all terms in
  $\sum_{t:\tau_\last(t)\ge i} \alpha_{\tau_\last(t)}^i$ will vanish
  if $\tau_\last(t)=t$. By the design of the triggering sequence
  $\set{t_n}$, we know that this happens for all $t\le \tau(r_\star)$,
  so we have
  \begin{align*}
    \wt{S}_i - S_i = \sum_{t:\tau_\last(t)\ge i; t>\tau(r_\star)}
    \alpha_{\tau_\last(t)}^i - \alpha_t^i.
  \end{align*}
  Let $r(i)=\min\set{r:\tau(r)\ge i}$, then the above can be rewritten
  as 
  \begin{align*}
    \wt{S}_i - S_i = \sum_{r\ge \max\set{r_\star, r(i)}}
    \sum_{t=\tau(r)}^{\tau(r+1)-1} \alpha_{\tau(r)}^i - \alpha_t^i.
  \end{align*}
  For each $t$ (and associated $r\ge r_\star$), we have the bound
  \begin{align*}
    & \quad \alpha_{\tau(r)}^i - \alpha_t^i =
    \alpha_t^i\left[\prod_{j=\tau(r)+1}^t(1-\alpha_j)^{-1} - 1\right]
      =\alpha_t^i\left[\prod_{j=\tau(r)+1}^t\left(1-\frac{H+1}{H+j}\right)^{-1}
      - 1\right]\\
    & = \alpha_t^i\left[\prod_{j=\tau(r)+1}^t\left(1 +
      \frac{H+1}{j-1}\right)- 1\right] 
       \leq  \alpha_t^i\left[\left(1 +
      \frac{H+1}{\tau(r)}\right)^{t-\tau(r)} - 1\right]\\ 
    & \leq \alpha_t^i\left[\left(1 +
      \frac{H+1}{\tau(r)}\right)^{\tau(r+1)-\tau(r)-1} - 1\right]
      \stackrel{\text{(i)}}{\le}
      \alpha_t^i\left[\left(1 + \frac{H+1}{\tau(r)}\right)^{\eta \tau(r)} -
      1\right] \\
    & \stackrel{\text{(ii)}}{\le} \alpha_t^i\left[e^{\eta (H+1)} - 1\right]
      \le \alpha_t^i\cdot 2\eta(H+1).
  \end{align*}  

  In the above,
  (i) holds as we have
  \begin{equation*}
    \tau(r+1) - 1 - \tau(r) = \ceil{(1+\eta)^{r+1}} - 1 -
    \ceil{(1+\eta)^r} \le (1+\eta)^{r+1} -
    (1+\eta)^r \le \eta\tau(r),
  \end{equation*}
  and (ii) holds whenever $\eta(H+1)\le 1/2$. Choosing
  \begin{align*}
    \eta = \frac{1}{2H(H+1)}~~~{\rm and}~~~r_\star =
    \ceil{\frac{\log(10H^2)}{\log(1+\eta)}} \le 
    8H^2\log(10H^2), 
  \end{align*}
  the above requirement will be satisfied. Therefore we have
  \begin{align*}
    \wt{S}_i - S_i \le 2\eta(H+1) \sum_{r\ge \max\set{r_\star, r(i)}}
    \sum_{t=\tau(r)}^{\tau(r+1)-1} \alpha_t^i  \le 2\eta(H+1)
    \sum_{t=i}^\infty \alpha_t^i = \frac{1}{H}S_i,
  \end{align*}
  and thus
  \begin{align*}
    \wt{S}_i \le \left(1+\frac{1}{H}\right)S_i \le 1 + \frac{3}{H}.
  \end{align*}
\end{proof}

We are now in position to prove the main theorem.

\begin{reptheorem}{theorem:ql-ucb2}[Q-learning with UCB2H, restated]
Choosing $\eta=\frac{1}{2H(H+1)}$ and
  $r_\star=\ceil{\frac{\log(10H^2)}{\log(1+\eta)}}$,
  with probability at least $1-p$, the regret of
  Algorithm~\ref{algorithm:ql-ucb2} is bounded by
  $O(\sqrt{H^4SAT \ell})$, where $\ell\defeq \log(SAT/p)$ is a log factor.
  Further, the local switching cost is bounded as
  $N_\sw\le O(H^3SA\log(K/A))$.
\end{reptheorem}

\begin{proof-of-theorem}[\ref{theorem:ql-ucb2}]
  The proof consists of two parts: upper bounding the regret, and
  upper bounding the local switching cost.
  
  \paragraph{Part I: Regret bound}
By Lemma~\ref{lemma:max-error}, we have
\begin{equation*}
  \wt{\delta}_h^k \le \left(\wt{Q}_h^{k'} - Q_h^{\pi_k} + \left[\wt{Q}_h^k -
      \wt{Q}_h^{k'}\right]_+ \right)(x_h^k, a_h^k).
\end{equation*}
Applying Lemma~\ref{lemma:optimistic-and-accurate} with the
$k'=k_{\tau_\last(t)+1}$-th episode (so that there are $\tau_\last(t)$
visitations to $(x_h^k, a_h^k)$ prior to the $k'$-th episode), we have
the bound
\begin{equation}
  \label{equation:one-step-error}
  \begin{aligned}
    & \quad \left(\wt{Q}_h^{k'} - Q_h^{\pi_k}\right)(x_h^k, a_h^k) \le
    \left(\wt{Q}_h^{k'} - Q_h^\star\right)(x_h^k, a_h^k) + 
    \left(Q_h^\star - Q_h^{\pi_k}\right)(x_h^k, a_h^k) \\
    & \le \alpha_{\tau_\last(t)}^0H
    + \sum_{i=1}^{\tau_\last(t)}
    \alpha_{\tau_\last(t)}^i\wt{\phi}_{h+1}^{k_i} +
    \beta_{\tau_\last(t)}  - 
    \wt{\phi}_{h+1}^k + \wt{\delta}_{h+1}^k + \xi_{h+1}^k,
  \end{aligned}
\end{equation}
where we recall that $\beta_t=2\sum_i
\alpha_t^ib_i=\Theta(\sqrt{H^3\ell/t})$ and $\xi_{h+1}^k\defeq [(
    \what{\P}_h^{k_i} - \P_h )(V_{h+1}^\star - V_{h+1}^{\pi_k}](x_h^k,a_h^k)$.
By Lemma~\ref{lemma:perturbation-bound}, the perturbation term
$[\wt{Q}_h^k - \wt{Q}_h^{k'}]_+$ can be bounded as
\begin{equation}
  \label{equation:perturbation-error}
  [\wt{Q}_h^k - \wt{Q}_h^{k'}]_+(x_h^k, a_h^k) \le \beta_t + 
  \sum_{i=\tau_\last(t)+1}^t \alpha_t^i \wt{\phi}_{h+1}^{k_i} + C\sqrt{\frac{H^3\ell}{t}}.
\end{equation}
We now study the effect of adding~\eqref{equation:perturbation-error} 
onto~\eqref{equation:one-step-error}.
The term $C\sqrt{H^3\ell/t}$ in~\eqref{equation:perturbation-error}
and~$\beta_{\tau_\last(t)}$ in~\eqref{equation:one-step-error} can be
both absorbed into $\beta_t$ (as
$\beta_t\ge 2\sqrt{H^3\ell/t}$ and $\beta_{\tau_\last(t)}\le
\sqrt{1+\eta}\beta_t$), so these together is bounded by $C'\beta_t$
where $C'$ is an absolute constant.  

Adding~\eqref{equation:perturbation-error}
onto~\eqref{equation:one-step-error}, we obtain
\begin{align*}
  \wt{\delta}_h^k
  \le \underbrace{\alpha_{\tau_\last(t)}^0 H}_{\rm I} +
  \underbrace{\sum_{i=1}^{\tau_\last(t)}\alpha_{\tau_\last(t)}^i\wt{\phi}_{h+1}^{k_i} +
  \sum_{i=\tau_\last(t)+1}^t\alpha_t^i
  \wt{\phi}_{h+1}^{k_i}}_{\rm II}  + C'\beta_t - 
  \wt{\phi}_{h+1}^k + \wt{\delta}_{h+1}^k + \xi_{h+1}^k.
\end{align*}
We now sum the above bound over $k\in[K]$. For term I, it equals $H$
only when $\tau_\last(t)=0$, which happens only if $t=0$, so the sum
over $k$ is upper bounded by $SAH$.


For term II, we consider the coefficient in front of
$\wt{\phi}_{h+1}^{k'}$ for each $k' \in [K]$ when summing over
$k$. Let $n_h^k$ denote the number of visitations to $(x_h^k, a_h^k)$
prior to the $k$-th episode. For each $k'$, $\wt{\phi}_{h+1}^{k'}$ is
counted if $i = n_h^{k'}$ and $(x_h^k, a_h^k) =
(x_h^{k'},a_h^{k'})$. We use $t$ to denote $n_h^{k}$, then an 
$\alpha_{\tau_\last(t)}^{n_h^{k'}}$ appears if $\tau_\last(t) \ge n_h^{k'}$, and
an $\alpha_{t}^{n_h^{k'}}$ appears if
$\tau_\last(t)+1\le n_h^{k'}\le t$. So the total coefficient in front
of $\wt{\phi}_{h+1}^{k'}$ is at most
\begin{equation*}
  \sum_{t:t\ge n_h^{k'}, \tau_\last(t)\le n_h^{k'}-1} \alpha_t^{n_h^{k'}} + \sum_{t:\tau_\last(t)\ge n_h^{k'}}
  \alpha_{\tau_\last(t)}^{n_h^{k'}},
\end{equation*}
for each $k' \in [K]$.
Choosing $\eta=\frac{1}{2H(H+1)}$ and
$r_\star=\ceil{\frac{\log(10H^2)}{\log(1+\eta)}}$, applying
Lemma~\ref{lemma:double-counting}, the above is upper bounded by
$1+3/H$.

For the remaining terms, we can adapt the proof of Theorem 1
in~\cite{jin2018q} and obtain a propagation of error inequality, and
deduce (as $(1+3/H)^H=O(1)$) that the regret is bounded by
$O(\sqrt{H^4SAT\ell})$. This concludes the proof.

\paragraph{Part II: Bound on local switching cost}
For each $(h,x)\in[H]\times\mc{S}$ and each action $a\in\mc{A}=[A]$,
either it is in stage I, which induces a switching cost of at most
$\tau(r_\star)$, or it is in stage II, which according to the
triggering sequence induces a switching cost of
\begin{equation*}
  \tau(r_\star) + r_a - r_\star \le \tau(r_\star) + r_a,
\end{equation*}
where $r_a$ is the final index for action $a$ satisfying
\begin{equation*}
  \sum_{a=1}^A \ceil{(1+\eta)^{r_a}} \le K + H,
\end{equation*}
(define $r_a=0$ if action $a$ has not reached the second stage.)
Applying Jensen's inequality gives that
\begin{align*}
  \sum_{a=1}^A r_a \le
  \frac{A\log((K+H)/A)}{\log(1+\eta)}
  = O\left( H^2A\log(K/A) \right)
\end{align*}
So the switching cost for $(h,x)$ can be bounded as
\begin{align*}
  & \quad A\tau(r_\star) + \sum_{a=1}^A r_a \\
  & \le A\ceil{(1+\eta)^{r_\star}} +
    O\left( H^2A\log(K/A) \right) \le A\ceil{(1+\eta)\cdot 10H^2} +
    O\left( H^2A\log(K/A) \right) \\
  & \le 20H^2A + O\left( H^2A\log(K/A) \right) =
    O\left(H^2A\log(K/A)\right).
\end{align*}
Multiplying the above by $HS$ (the number of $(h,x)$ pairs) gives the
desired bound.
\end{proof-of-theorem}


\section{Q-learning with UCB2-Bernstein exploration}
\label{appendix:proof-ql-bernstein}

\subsection{Algorithm description}
We present the algorithm, Q-Learning with UCB2-Bernstein (UCB2B)
exploration, in Algorithm~\ref{algorithm:ql-ucb2-Bernstein} below.

\begin{algorithm*}[ht]
   \caption{Q-learning with UCB2-Bernstein (UCB2B) Exploration}
   \label{algorithm:ql-ucb2-Bernstein}
\begin{algorithmic}
   \INPUT Parameter $\eta\in(0,1)$, $r_\star\in\Z_{>0}$, and $c>0$.
   \STATE {\bfseries Initialize:} $\wt{Q}_h(x,a)\setto H$,
   $Q_h\setto \wt{Q}_h$, $N_h(x,a)\setto 0$ for all
   $(x,a,h)\in\mc{S}\times \mc{A}\times [H]$.
   \FOR{episode $k=1,\dots,K$}
   \STATE Receive $x_1$.
   \FOR{step $h=1,\dots, H$}
   \STATE Take action $a_h\setto \argmax_{a'}Q_h(x_h,a')$, and observe
   $x_{h+1}$.
   \STATE $t=N_h(x_h, a_h)\setto N_h(x_h, a_h) + 1$.
   \STATE $\mu_h(x_h, a_h) \setto \mu_h(s_h, a_h) + \wt{V}_{h + 1}(x_{h + 1})$.
   \STATE $\sigma_h(x_h, a_h) \setto \sigma_h(x_h, a_h) + \left( \wt{V}_{h + 1}(x_{h + 1}) \right)^2$.
   \STATE $W_t(x_h,a_h,h) = \frac{1}{t}\left(\sigma_h(x_h, a_h) - \left(\mu_h(x_h, a_h)\right)^2 \right)$.
   \STATE $\beta_t(x_h,a_h,h) \setto \min \left\{ c_1 \left( \sqrt{\frac{H}{t}(W_t(x_h,a_h,h) + H)\ell} + \frac{\sqrt{H^7 S A}\cdot\ell}{t} \right), c_2 \sqrt{\frac{H^3 \ell}{t}} \right\}$.
   \STATE $b_t \setto \frac{\beta_t(x_h,a_h,h) - (1 -
     \alpha_t)\beta_{t - 1}(x_h,a_h,h)}{2 \alpha_t}$ (Bernstein-type bonus).
   \STATE $\wt{Q}_h(x_h, a_h)\setto (1-\alpha_t)\wt{Q}_h(x_h, a_h) +
   \alpha_t[r_h(x_h, a_h) + \wt{V}_{h+1}(x_{h+1}) + b_t]$.
   \STATE $\wt{V}_h(x_h)\setto \min\set{H, \max_{a'\in\mc{A}}\wt{Q}_h(x_h,
     a')}$.
   {\color{red!80!black}
   \IF{$t \in \set{t_n}_{n\ge 1}$ (where $t_n$ is defined
     in~\eqref{equation:switching-sequence})}
   \STATE (Update policy) $Q_h(x_h,\cdot) \setto \wt{Q}_h(x_h,\cdot)$.
   \ENDIF}
   \ENDFOR
   \ENDFOR
\end{algorithmic}
\end{algorithm*}

\subsection{Proof of Theorem~\ref{theorem:ql-ucb2-Bernstein}}

We first present the analogs of Lemmas that we used in the proof of Theorem \ref{theorem:ql-ucb2}.

\begin{lemma}[$\wt{Q}$ is optimistic and accurate for the Bernstein case; Lemma C.1 \&
  C.4,~\cite{jin2018q}]
  \label{lemma:optimistic-and-accurate-Bernst}
  We have for all $(h,x,a,k)\in[H]\times\mc{S}\times\mc{A}\times[K]$
  that
  \begin{equation}
    \label{equation:q-diff-Bernst}
  \begin{aligned}
    & \wt{Q}_h^k(x,a) - Q_h^\star(x,a) \\
     = & \alpha_t^0(H - Q_h^\star(x,a)) + \\
    & \sum_{i=1}^t \alpha_t^i\bigg( r_h(x,a) +
    \wt{V}_{h+1}^{k_i}(x_{h+1}^{k_i}) - V^\star_{h+1}(x_{h+1}^{k_i})
    + \left[\left( \what{\P}_h^{k_i} - \P_h \right)V_{h+1}^\star\right](x,a) +
    b_i\bigg),
    \end{aligned}
  \end{equation}
  where $[\widehat{\P}_h^{k_i} V_{h + 1}](x,a) \defeq V_{h + 1}(x_{h+1}^{k_i})$.

  Further, with probability at least $1-p$, under the choice of
  $b_t$ and $\beta_t$ in Algorithm \ref{algorithm:ql-ucb2-Bernstein}, we have for all $(h,x,a,k)$ that
  \begin{align*}
    & 0\le \wt{Q}_h^k(x,a) - Q_h^\star(x,a)
      \le \alpha_t^0H +
      \sum_{i=1}^t\alpha_t^i(\wt{V}_{h+1}^{k_i} -
      V_{h+1}^\star)(x_{h+1}^{k_i}) + \beta_t.
  \end{align*}
\end{lemma}

The following Lemma is the analog of Lemma~\ref{lemma:max-error} in the Bernstein case.
\begin{lemma}[Max error under delayed policy update]
  \label{lemma:max-error-Bernst}
  We have
  \begin{equation}
    \label{equation:delayed-behavior-Bernst}
    \begin{aligned}
      \wt{\delta}_h^k
      \le \left(\max\set{\wt{Q}_h^{k'}, \wt{Q}_h^k} -
        Q_h^{\pi_k}\right)(x_h^k, a_h^k) = \left(\wt{Q}_h^{k'} -
        Q_h^{\pi_k} + \left[\wt{Q}_h^{k} - 
          \wt{Q}_h^{k'}\right]_+ \right)(x_h^k, a_h^k).
    \end{aligned}
  \end{equation}
  where $k'=k_{\tau_\last(t)+1}$ (which depends on $k$.) In particular,
  if $t=\tau_\last(t)$, then $k=k'$ and the upper bound reduces to
  $(\wt{Q}_h^{k'} - Q_h^{\pi_k})(x_h^k, a_h^k)$.
\end{lemma}
The proof of Lemma~\ref{lemma:max-error-Bernst} can be adapted from the proof of Lemma~\ref{lemma:max-error}. The following Lemma is the analog of Lemma~\ref{lemma:perturbation-bound} in the Bernstein case.
\begin{lemma}[Perturbation bound on $(\wt{Q}_h^k - \wt{Q}_h^{k'})_+$]
  \label{lemma:perturbation-bound-Bernst}
  For any $k$ such that $k>k'$ (so that the perturbation term is
  non-zero), we have
  \begin{equation}
    \label{equation:perturbation-bound-Bernst}
    \begin{aligned}
      &  \left[\wt{Q}_h^{k} - \wt{Q}_h^{k'}\right]_+(x_h^k, a_h^k) \le
      \beta_t +
      \sum_{i=\tau_\last(t)+1}^t\alpha_t^i\wt{\phi}_{h+1}^{k_i} +
      \overline{\zeta}_h^k,
    \end{aligned}
  \end{equation}
  where
  \begin{align*}
    \overline{\zeta}_h^k \defeq \left|\sum_{i=\tau_\last(t)+1}^t
    \alpha_t^i[(\what{\P}_h^k - \P_h)V_{h+1}^\star](x_h^k,
    a_h^k) \right|.
  \end{align*}
\end{lemma}
The proof of Lemma~\ref{lemma:perturbation-bound-Bernst} can be adapted from the proof of Lemma~\ref{lemma:perturbation-bound}, but we used a finer bound on the summation $\overline{\zeta}_h^k$ over $k \in [K]$ in the proof of Theorem \ref{theorem:ql-ucb2-Bernstein}.

\begin{lemma}[Variance is bounded and $W_t$ is accurate; Lemma C.5 \&
  C.6,~\cite{jin2018q}]
\label{lemma:variance-bound-Bernst}
There exists an absolute constant $c$, such that
\begin{align}
\sum_{k = 1}^{K}\sum_{h = 1}^{H} \mathbb V_h V_{h+1}^{\pi_k}(x_h^k,a_h^k) \leq c(HT + H^3 \ell),
\end{align}
w.p. at least $(1 - p)$.

Further, w.p. at least $(1 - 4p)$, there exists an absolute constant $c > 0$ such that, letting $(x,a) = (x_h^k, a_h^k)$ and $t = n_h^k = N_h^k(x,a)$, we have
\begin{align}
W_t(x,a,h) \leq \mathbb V_h V_{h+1}^{\pi_k}(x,a) + 2H(\wt{\delta}_{h+1}^k + \xi_{h+1}^k) + c\left( \frac{SA\sqrt{H^7 \ell}}{t} + \sqrt{\frac{SAH^7 \ell}{t}} \right)
\end{align}
for all $(k,h) \in [K] \times [H]$, where the variance operator $\mathbb V_h$ is defined by
\begin{align}
[\mathbb V_h V_{h+1}](x,a) \coloneqq \var_{x' \sim \mathbb P_h(\cdot|x,a)}(V_{h+1}(x')) = \mathbb E_{x' \sim \mathbb P_h(\cdot|x,a)}\left[V_{h+1}(x') - [\mathbb P_h V_{h+1}](x,a)\right]^2.
\end{align}

\end{lemma}

Now, it is ready to present the proof of Theorem \ref{theorem:ql-ucb2-Bernstein}.

\begin{reptheorem}{theorem:ql-ucb2-Bernstein}[Q-learning with UCB2B, restated]
Choosing $\eta=\frac{1}{2H(H+1)}$ and
  $r_\star=\ceil{\frac{\log(10H^2)}{\log(1+\eta)}}$,
  with probability at least $1-p$, the regret of
  Algorithm~\ref{algorithm:ql-ucb2-Bernstein} is bounded by
  $O(\sqrt{H^3 S A T \ell^2} + \sqrt{S^3 A^3 H^9 \ell^4})$, where $\ell\defeq \log(SAT/p)$ is a log factor. Further, the local switching cost is bounded as $N_\sw\le O(H^3SA\log(K/A))$.
\end{reptheorem}

\begin{proof-of-theorem}[\ref{theorem:ql-ucb2-Bernstein}]

By Lemma~\ref{lemma:max-error-Bernst}, we have
\begin{equation*}
  \wt{\delta}_h^k \le \left(\wt{Q}_h^{k'} - Q_h^{\pi_k} + \left[\wt{Q}_h^k -
      \wt{Q}_h^{k'}\right]_+ \right)(x_h^k, a_h^k).
\end{equation*}
Applying Lemma~\ref{lemma:optimistic-and-accurate-Bernst} with the
$k'=k_{\tau_\last(t)+1}$-th episode (so that there are $\tau_\last(t)$
visitations to $(x_h^k, a_h^k)$ prior to the $k'$-th episode), we have
the bound
\begin{equation}
  \label{equation:one-step-error-Bernst}
  \begin{aligned}
    & \quad \left(\wt{Q}_h^{k'} - Q_h^{\pi_k}\right)(x_h^k, a_h^k) \le
    \left(\wt{Q}_h^{k'} - Q_h^\star\right)(x_h^k, a_h^k) + 
    \left(Q_h^\star - Q_h^{\pi_k}\right)(x_h^k, a_h^k) \\
    & \le \alpha_{\tau_\last(t)}^0H
    + \sum_{i=1}^{\tau_\last(t)}
    \alpha_{\tau_\last(t)}^i\wt{\phi}_{h+1}^{k_i} +
    \beta_{\tau_\last(t)}  - 
    \wt{\phi}_{h+1}^k + \wt{\delta}_{h+1}^k + \xi_{h+1}^k.
  \end{aligned}
\end{equation}
By Lemma~\ref{lemma:perturbation-bound-Bernst}, the perturbation term
$[\wt{Q}_h^k - \wt{Q}_h^{k'}]_+$ can be bounded as
\begin{equation}
  \label{equation:perturbation-error-Bernst}
  [\wt{Q}_h^k - \wt{Q}_h^{k'}]_+(x_h^k, a_h^k) \le \beta_t + 
  \sum_{i=\tau_\last(t)+1}^t \alpha_t^i \wt{\phi}_{h+1}^{k_i} + \overline{\zeta}_h^k.
\end{equation}
Thus, adding~\eqref{equation:perturbation-error-Bernst}
onto~\eqref{equation:one-step-error-Bernst}, we obtain
\begin{align}
  \wt{\delta}_h^k \le & \underbrace{\alpha_{\tau_\last(t)}^0 H}_{\rm I} +
  \underbrace{\sum_{i=1}^{\tau_\last(t)}\alpha_{\tau_\last(t)}^i\wt{\phi}_{h+1}^{k_i} +
  \sum_{i=\tau_\last(t)+1}^t\alpha_t^i
  \wt{\phi}_{h+1}^{k_i}}_{\rm II} + \underbrace{\overline{\zeta}_h^k}_{\rm III} \\
  & + \underbrace{\beta_{\tau_\last(t)}}_{\rm IV} + \underbrace{\xi_{h+1}^k}_{V} - 
  \wt{\phi}_{h+1}^k + \wt{\delta}_{h+1}^k + \beta_t.
\end{align}

We now sum the above bound over $k\in[K]$ and $h \in [H]$. For term I, it equals $H$ only when $\tau_\last(t)=0$, which happens only if $t=0$, so the sum over $k$ is upper bounded by $SAH$.

For term II, we follow the same argument in the proof of Theorem \ref{theorem:ql-ucb2} and obtain:
\begin{align}
\sum_{k = 1}^{K}\left( \sum_{i=1}^{\tau_\last(t)}\alpha_{\tau_\last(t)}^i\wt{\phi}_{h+1}^{k_i} + \sum_{i=\tau_\last(t)+1}^t\alpha_t^i \wt{\phi}_{h+1}^{k_i} \right) \leq \left( 1+ \frac{3}{H}\right) \sum_{k = 1}^{K} \wt{\phi}_{h + 1}^k
\end{align}

For term III, we first apply the Azuma-Hoeffding inequality to get that
\begin{align}
  & \overline{\zeta}_h^k \le c\sqrt{\sum_{i=\tau_\last(t)+1}^t
    (\alpha_t^i)^2 H^2\ell}
\end{align}
uniformly over $(h,k)$, then we sum it the above  over $k\in[K]$, and then we obtain
\begin{align}
  \sum_{k = 1}^{K} \overline{\zeta}_h^k \le & cH\sqrt{\ell}\sum_{k = 1}^{K} \sqrt{\sum_{i=\tau_\last(t)+1}^t
    (\alpha_t^i)^2} \leq cH\sqrt{\ell}\sum_{k = 1}^{K} \sqrt{\sum_{i=\ceil{\frac{n_h^k}{1+\eta}}}^{n_h^k}
    \left(\alpha_{n_h^k}^i\right)^2} \\
    \leq & cH\sqrt{\ell}\sum_{k = 1}^{K} \sqrt{\left(n_h^k - \ceil{\frac{n_h^k}{1+\eta}}\right)\left(\max_{i \in [n_h^k]}\alpha_{n_h^k}^i\right)^2} \\
    \leq & cH\sqrt{\ell}\sum_{k = 1}^{K} \sqrt{\eta n_h^k\frac{4H^2}{(n_h^k)^2}} \\
    \leq & cH\sqrt{\ell}\sum_{k = 1}^{K} \sqrt{\frac{1}{n_h^k}} \stackrel{\text{(i)}}{=} cH\sqrt{\ell} \sum_{x,a} \sum_{n = 1}^{N_h^k(s,a)}\sqrt{\frac{1}{n}} \stackrel{\text{(ii)}}{\le} cH \sqrt{SAK\ell},
    \label{eq:Bernst_term3_lastline}
\end{align}
where (i) follows the fact $\sum_{s,a}N_h^K(x,a) = K$, and (ii) follows the property that the LHS of (ii) is maximized when $N_h^K(x,a) = K/SA$ for all $x,a$.

For term IV, we have
\begin{align}
\label{eq:Bernst_term_4}
\sum_{k = 1}^{K}\sum_{h = 1}^H \beta_{\tau_\last(n_h^k)} \leq c_1 \sum_{k = 1}^{K}\sum_{h = 1}^H\left( \sqrt{\frac{H}{\tau_\last(n_h^k)}(W_{\tau_\last(n_h^k)}(x,a,h) + H)\ell} + \frac{\sqrt{H^7 S A}\cdot\ell}{\tau_\last(n_h^k)} \right)
\end{align}
by our choice of $\beta_t$ in Algorithm \ref{algorithm:ql-ucb2-Bernstein}.
We first upper bound summation the $W_{\tau_\last(n_h^k)}(x,a,h)$ term as follows
\begin{align}
& \sum_{k = 1}^{K}\sum_{h = 1}^H W_{\tau_\last(n_h^k)}(x,a,h) \\
\stackrel{\text{(i)}}{\le} & \sum_{k = 1}^{K}\sum_{h = 1}^H \left[ \mathbb V_h V_{h+1}^{\pi_k}(x,a) + 2H(\delta_{h+1}^k + \xi_{h+1}^k) + c\left( \frac{SA\sqrt{H^7 \ell}}{\tau_\last(n_h^k)} + \sqrt{\frac{SAH^7 \ell}{\tau_\last(n_h^k)}} \right) \right]\\
\stackrel{\text{(ii)}}{\le} & \sum_{k = 1}^{K}\sum_{h = 1}^H \left[ \mathbb V_h V_{h+1}^{\pi_k}(x,a) + 2H(\delta_{h+1}^k + \xi_{h+1}^k) + c(1+\eta)\left( \frac{SA\sqrt{H^7 \ell}}{n_h^k} + \sqrt{\frac{SAH^7 \ell}{n_h^k}} \right)\right] \\
\stackrel{\text{(iii)}}{\le} & \sum_{k = 1}^{K}\sum_{h = 1}^H \left[ \mathbb V_h V_{h+1}^{\pi_k}(x,a) + 2H(\delta_{h+1}^k + \xi_{h+1}^k) \right] + c(1+\eta)\left( S^2 A^2 \sqrt{H^9 \ell^3} + SA \sqrt{H^8 T \ell} \right) \\
\stackrel{\text{(iv)}}{\le} & 2H \sum_{k = 1}^{K}\sum_{h = 1}^H (\delta_{h+1}^k + \xi_{h+1}^k) + c'\left(HT + H^3 \ell + S^2 A^2 \sqrt{H^9 \ell^3} + SA \sqrt{H^8 T \ell}\right),
\label{eq:bound_w_terms}
\end{align}
where inequalities (i) and (iv) follow from Lemma~\ref{lemma:variance-bound-Bernst}, inequality (ii) follows from $\tau_\last(n_h^k) \geq n_h^k/(1 + \eta)$, and inequality (iii) uses the properties that $\sum_{k = 1}^K (n_h^k)^{-1}$ and $\sum_{k = 1}^K (n_h^k)^{-1/2}$ are maximized when $N_h^K(x,a) = K/SA$ for all $x,a$ (similar to \eqref{eq:Bernst_term3_lastline}).

We now consider the first term in \eqref{eq:bound_w_terms}. By the Azuma-Hoeffding inequality, we have
\begin{align}
\left| \sum_{h' = h}^{H}\sum_{k = 1}^{K} \xi_{h'+1}^k \right| \leq \left| \sum_{h' = h}^{H}\sum_{k = 1}^{K} [(
    \what{\P}_{h'}^{k_i} - \P_h )(V_{h'+1}^\star - V_{h'+1}^{\pi_k})](x_{h'}^k,a_{h'}^k) \right| \leq O(H\sqrt{T \ell}),
\label{eq:Bernst_initialbound_xi}
\end{align}
w.p. $1 - p$ for all $h \in [H]$. Recall $\beta_t(x,a,h) \leq c\sqrt{H^3 \ell / t}$, we can simply obtain
\begin{align}
\sum_{k = 1}^K \delta_h^k \leq O(\sqrt{H^4 SAT \ell}),
\label{eq:Bernst_initialbound_delta}
\end{align}
for all $h \in [H]$ by adapting the proof of Theorem \ref{theorem:ql-ucb2}. Then, using \eqref{eq:Bernst_initialbound_xi} and \eqref{eq:Bernst_initialbound_delta}, we obtain the upper bound of the summation of $W_{\tau_\last(n_h^k)}(x,a,h)$ term for $h \in [H]$ and $k \in [K]$
\begin{align}
& \sum_{k = 1}^{K}\sum_{h = 1}^H W_{\tau_\last(n_h^k)}(x,a,h) \\
& 2H \sum_{k = 1}^{K}\sum_{h = 1}^H (\delta_{h+1}^k + \xi_{h+1}^k) + c'\left(HT + H^3 \ell + S^2 A^2 \sqrt{H^9 \ell^3} + SA \sqrt{H^8 T \ell}\right) \\
\leq & O\left( HT + S^2 A^2 H^7 \ell + S^2 A^2 \sqrt{H^9 \ell^3} \right).
\label{eq:upper_bound_w}
\end{align}

Now it is ready to upper bounded the summation of the first term in \eqref{eq:Bernst_term_4},
\begin{align}
& \sum_{k = 1}^{K}\sum_{h = 1}^H \sqrt{\frac{H}{\tau_\last(n_h^k)}(W_{\tau_\last(n_h^k)}(x,a,h) + H)\ell} \\
\stackrel{\text{(i)}}{\le} & \sqrt{\left( \sum_{k = 1}^{K}\sum_{h = 1}^H (W_{\tau_\last(n_h^k)}(x,a,h) + H) \right) \left(\sum_{k = 1}^{K}\sum_{h = 1}^H  \frac{H}{\tau_\last(n_h^k)} \right) \ell} \\
\stackrel{\text{(ii)}}{\le} & (1 + \eta) \sqrt{\sum_{k = 1}^{K}\sum_{h = 1}^H W_{\tau_\last(n_h^k)}(x,a,h)} \cdot \sqrt{H^2 S A \ell^2} + (1 + \eta) \sqrt{H^3 SAT \ell^2} \\
\stackrel{\text{(iii)}}{\le} & O(\sqrt{H^3 SAT \ell^2})
\label{eq:first_term_beta}
\end{align}
where inequality (i) follows from the Cauchy–Schwarz inequality, inequality (ii) follows from the facts that $\tau_\last(n_h^k) \geq n_h^k/(1 + \eta)$ and $\sum_{k = 1}^K (n_h^k)^{-1}$ is maximized when $N_h^K(x,a) = K/SA$ for all $x,a$, and inequality (iii) follows from \eqref{eq:upper_bound_w}.

The summation of the second term in \eqref{eq:Bernst_term_4} can be upper bounded by
\begin{align}
\sum_{k = 1}^{K}\sum_{h = 1}^H \frac{\sqrt{H^7 S A}\cdot\ell}{\tau_\last(n_h^k)} \leq \sum_{k = 1}^{K}\sum_{h = 1}^H \frac{(1 + \eta)\sqrt{H^7 S A}\cdot\ell}{n_h^k} \leq (1 + \eta)\sqrt{H^9 S^3 A^3 \ell^4},
\label{eq:second_term_beta}
\end{align}
by following $\tau_\last(n_h^k) \geq n_h^k/(1 + \eta)$ and $1 + 1/2 + 1/3 + \cdots \leq \ell$. 

Putting \eqref{eq:Bernst_term_4}, \eqref{eq:first_term_beta}, and \eqref{eq:second_term_beta} together, we have
\begin{align}
\sum_{k = 1}^{K}\sum_{h = 1}^H \beta_{\tau_\last(n_h^k)} \leq O\left(\sqrt{H^3 S A T \ell^2} + \sqrt{S^3 A^3 H^9 \ell^4}\right).
\end{align}

For the remaining terms, we can adapt the proof of Theorem 2 in~\cite{jin2018q} and obtain a propagation of error inequality. Thus, we deduce that the regret is bounded by $O(\sqrt{H^3 S A T \ell^2} + \sqrt{S^3 A^3 H^9 \ell^4})$. The bound on local switching cost can be adapted from the proof of Theorem \ref{theorem:ql-ucb2}. This concludes the proof.
\end{proof-of-theorem}
\section{Proof of Corollary~\ref{corollary:pac-bound}}
\label{appendix:proof-pac-bound}

Consider first Q-Learning with UCB2H exploration. By
Theorem~\ref{theorem:ql-ucb2}, we know that the regret is bounded by
$\wt{O}(\sqrt{H^4SAT})$ with high probability, that is, we have
\begin{equation*}
  \sum_{k=1}^K V_1^\star(x_1) - V_1^{\pi_k}(x_1) \le
  \wt{O}(\sqrt{H^4SAT}).
\end{equation*}
Now, define a stochastic policy $\what{\pi}$ as
\begin{equation*}
  \what{\pi} = \frac{1}{K}\sum_{k=1}^K \pi_k.
\end{equation*}
By definition we have
\begin{equation*}
  \E\left[V_1^\star(x_1) - V_1^{\what{\pi}}(x_1)\right] =
  \frac{1}{K}\sum_{k=1}^K \left[V_1^\star(x_1) - V_1^{\pi_k}(x_1)
  \right] \le \wt{O}\left(\frac{\sqrt{H^4SAT}}{K}\right) = \wt{O}\left(
    \sqrt{\frac{H^5SA}{K}} \right).
\end{equation*}
So by the Markov inequality, we have with high probability that
\begin{equation*}
  V_1^\star(x_1) - V_1^{\what{\pi}}(x_1) \le \wt{O}\left(
    \sqrt{\frac{H^5SA}{K}} \right).
\end{equation*}
Taking $K=\wt{O}(H^5SA/\eps^2)$ bounds the above by $\eps$.

For Q-Learning with UCB2B exploration, the regret bound is
$\wt{O}(\sqrt{H^3SAT})$. A similar argument as above gives that
$K=\wt{O}(H^4SA/\eps^2)$ episodes guarantees an $\eps$ near-optimal
policy with high probability. \qed
\section{Proof of Theorem~\ref{theorem:concurrent}}
\label{appendix:concurrent}
We first present the concurrent version of low-switching cost
Q-learning with \{UCB2H, UCB2B\} exploration.

\paragraph{Algorithm description}
At a high level, our algorithm is a very intuitive parallelization of
the vanilla version -- we ``parallelize as much as you can'' until we
have to switch.

More concretely, suppose the policy $Q_h$ has been switched $(t-1)$
times and we have a new policy yet to be executed. We execute this
policy on all $M$ machines, and read the observed trajectories from
machine $1$ to $M$ to determine a number $m\in\set{1,\dots,M}$ such
that the policy needs to be switched (according to the UCB2 schedule)
after $m$ episodes. We then only keep the data on machines
$1,\dots,m$, use them to compute the next policy, and throw away all
the rest of the data on machines $m+1,\dots,M$. The full
algorithm is presented in
Algorithm~\ref{algorithm:concurrent-ql-ucb2}.



\begin{algorithm*}[ht]
   \caption{Concurrent Q-learning with UCB2 scheduling}
   \label{algorithm:concurrent-ql-ucb2}
\begin{algorithmic}
  \INPUT One of the UCB2-\{Hoeffding, Bernstein\} bonuses for updating
  $\wt{Q}$.  \STATE {\bfseries Initialize:} $\wt{Q}_h(x,a)\setto H$,
  $Q_h\setto \wt{Q}_h$, $t\setto 1$. 
   \WHILE{stopping criterion not satisfied}
   \FOR{rounds $r_t=1,2,\dots$}
   \STATE Play according to $Q_h$ concurrently on all $M$
   machines and store the trajectories.
   \STATE Aggregate the trajectories and feed them sequentially into
   the UCB2 scheduling to determine whether a switch is needed.
   \IF{Switch is needed after $m\in\set{1,\dots,M}$ episodes}
   \STATE {\bf BREAK}
   \ENDIF
   \ENDFOR
   \STATE Update the policy $\wt{Q}_h$ from all the $M(r_t-1)+m$
   stored trajectories using \{Hoeffding, Bernstein\} bonus.
   \STATE Set $Q_h(\cdot,\cdot)\setto \wt{Q}_h(\cdot,\cdot)$ and
   $t\setto t+1$.
   \ENDWHILE
\end{algorithmic}
\end{algorithm*}

\subsection{Proof of Theorem~\ref{theorem:concurrent}}
The way that Algorithm~\ref{algorithm:concurrent-ql-ucb2} is
constructed guarantees that its execution path is \emph{exactly
  equivalent} (i.e. equal in distribution) to the execution path of
the vanilla non-parallel Q-Learning with UCB2\{H, B\} exploration,
except that it does not fully utilize the data on all $M$ machines and
needs to throw away some data. As a corollary, if the non-parallel
version plays $L_t$ episodes in between the $(t-1)$-th and $t$-th
switch, then the parallel/concurrent version will play the same
episodes in $\ceil{L_t/M}$ rounds.

Now, suppose we wish to play a total of $K$ episodes concurrently with
$M$ machines, and the corresponding non-parallel version of Q-learning
is guaranteed to have at most $N_{\sw}$ local switches with $L_t$
episodes played before each switch. Let $R$ denote the total number of
rounds, then we have
\begin{equation*}
  R = \sum_{t=1}^{N_{\sw}} r_t = \sum_{t=1}^{N_{\sw}}
  \ceil{\frac{L_t}{M}} \le \sum_{t=1}^{N_{\sw}} \left(1 +
    \frac{L_t}{M}\right) \le N_{\sw} + \frac{K}{M}.
\end{equation*}
Now, to find $\eps$ near-optimal policy, we know by
Corollary~\ref{corollary:pac-bound} that Q-learning with
\{UCB2H, UCB2B\} exploration requires at most
\begin{equation*}
  K = O\left( \frac{H^{\set{5,4}}SA\log(HSA)}{\eps^2} \right)
\end{equation*}
episodes. Further, choosing $K$ as above, by
Theorem~\ref{theorem:ql-ucb2} and~\ref{theorem:ql-ucb2-Bernstein},
the switching cost is bounded as
\begin{equation*}
  N_{\sw} \le O\left(H^3SA\log(K/A)\right) =
  O\left(H^3SA\log(HSA/\eps)\right).
\end{equation*}
Plugging these into the preceding bound on $R$ yields
\begin{equation*}
  R \le O\left( H^3SA\log(HSA/\eps) +
    \frac{H^{\set{5,4}}SA\log(HSA)}{\eps^2M} \right) = \wt{O}\left(
    H^3SA + \frac{H^{\set{5,4}}SA}{\eps^2M} \right),
\end{equation*}
the desired result. \qed

\subsection{Concurrent algorithm with mistake bound}
\label{appendix:concurrent-mistake-bound}
Our concurrent algorithm (Algorithm~\ref{algorithm:concurrent-ql-ucb2}
can be converted straightforwardly to an algorithm with low mistake
bound. Indeed, for any given $\eps$, by
Theorem~\ref{theorem:concurrent}, we obtain an $\eps$ near-optimal
policy with high probability by running
Algorithm~\ref{algorithm:concurrent-ql-ucb2} for
\begin{equation*}
  \wt{O}\left(H^3SA + \frac{H^{\set{5,4}}SA}{\eps^2M} \right)
\end{equation*}
rounds. We then run this $\eps$ near-optimal policy forever and are
guaranteed to make no mistake.

For such an algorithm, with high probability, ``mistakes'' can only
happen in the exploration phase. Therefore the total amount of
``mistakes'' (performing an $\eps$ sub-optimal action) is upper
bounded by the above number of exploration rounds multiplied by $HM$,
as each round consists of at most $M$ machines\footnote{To have a fair
  comparison with CMBIE, if a round does not utilize all $M$ machines,
  we still let all $M$ machines run and count their actions as their
  ``mistakes''.}  each performing $H$ actions. This yields a mistake
bound
\begin{equation*}
  \wt{O}\left(H^4SAM + \frac{H^{\set{6,5}}SA}{\eps^2} \right)
\end{equation*}
as desired.
\section{Proof of Theorem~\ref{theorem:lower-bound}}
\newcommand{\Unif}{{\rm Unif}}
\label{appendix:proof-lower-bound}
Recall that $\mc{M}$ denotes the set of all MDPs with horizon $H$,
state space $S$, action space $A$, and deterministic rewards in
$[0,1]$.  Let $K$ be the number of episodes that we can run, and
$\mc{A}$ be any RL algorithm satisfying that
\begin{equation*}
  N_{\sw} = \sum_{(h,x)}n_{\sw}(h, x) \le HSA/2
\end{equation*}
almost surely. We want to show that
\begin{equation*}
  \sup_{M\in\mc{M}} \E_{x_1, M}\left[\sum_{k=1}^K
  V_1^\star(x_1) - V_1^{\pi_k}(x_1)\right] \ge \Omega(K),
\end{equation*}
i.e. the worst case regret is linear in $K$.


\subsection{Construction of prior}
Let $a^\star:[H]\times[S]\to[A]$ denote a mapping that maps each
$(h,x)$ to an action $a^\star(h,x)\in[A]$. There are $A^{HS}$ such
mappings. For each $a^\star$, define an MDP $M_{a^\star}$ where the
transition is uniform, i.e.
\begin{equation*}
  x_1\sim\Unif([S]),~~x_{h+1}|x_h=x, a_h=a\sim\Unif([S])~~\textrm{for
    all}~(x,a)\in[S]\times[A],~h\in[H]
\end{equation*}
and the reward is 1 if $a_h=a^\star(h, x_h)$ and 0 otherwise, that is,
\begin{equation*}
  r_h(x, a) = \indic{a = a^\star(h, x)}.
\end{equation*}
Essentially, $M_{a^\star}$ is just a $H$-fold connection of $S$
parallel bandits that are $A$-armed, where $a^\star(h,x)$ is the only 
optimal action at each $(h,x)$.

For such MDPs, as the transition does not depend on the policy, the
value functions can be expressed explicitly as
\begin{equation*}
  \E_{x_1}[V_1^\pi(x_1)] = \frac{1}{S}\sum_{(h,x)\in[H]\times[S]}
  \indic{\pi_h(x)=a^\star(h,x)},
\end{equation*}
and we clearly have
\begin{equation*}
  \E_{x_1}[V_1^\star(x_1)] \equiv H.
\end{equation*}

\subsection{Minimax lower bound}
Using the sup to average reduction with the above prior, we have
the bound
\begin{align*}
  & \quad \sup_{M\in\mc{M}} \E_{x_1, M}\left[\sum_{k=1}^K
    V_1^\star(x_1) - V_1^{\pi_k}(x_1)\right] \ge
    \E_{a^\star} \E_{M_{a^\star}} \left[ KH -
    \sum_{k=1}^K V_1^{\pi_k}(x_1) \right] \\
  & = KH - \sum_{k=1}^K \E_{a^\star,M_{a^\star}}[V_1^{\pi_k}(x_1)].
\end{align*}
It remains to upper bound $\E_{a^\star,M_{a^\star}}[V_1^{\pi_k}(x_1)]$
for each $k$.

For all $k\ge 1$, let
\begin{equation*}
  n_\sw^k(h,x) \defeq \sum_{j=1}^{k-1} \indic{\pi_j^h(x) \neq
    \pi_{j+1}^h(x)}~~~{\rm and}~~~N_\sw^k = \sum_{h,x} n_\sw^k(h,x)
\end{equation*}
denote respectively the switching cost at a single $(h,x)$ and the
total (local) switching cost. We use the switching cost to upper bound
$\E_{a^\star, M_{a^\star}}[V_1^{\pi_k}]$.

Let
\begin{equation*}
  A_k(h,x) \defeq \set{\pi_1^h(x),\dots,\pi_k^h(x)} \subseteq [A]
\end{equation*}
denote the set of visited actions at timestep $h$ and state $x$.
Observe that
\begin{align*}
  \E_{a^\star, M_{a^\star}}[V_1^{\pi_k}] =
  \frac{1}{S}\sum_{h,x}\E\left[\indic{a^\star(h,x)=\pi_k^h(x)}\right]
  \le \frac{1}{S}\sum_{h,x}\underbrace{\E[\indic{a^\star(h,x)\in
  A_k(h,x)}]}_{\defeq \Phi_k(h,x)}.
\end{align*}
Therefore it suffices to bound $\Phi_k(h,x)$.

It is clear that algorithms that only switch to unseen actions can
maximize the value function, so we henceforth restrict attention on
these algorithms. Let $a^\star=a^\star(h,x)$ and
$n_{\sw}^k=n_{\sw}^k(h,x)$ for convenience. Let
\begin{equation*}
  A_k(h,x) = \set{a^1,a^2,\dots,a^{n_\sw^k+1}}
\end{equation*}
be the ordered set of unique actions that have been taken at $(h,x)$
throughout the execution of the algorithm. We have
\begin{align*}
  & \quad \Phi_k(h, x) = \P(a^\star\in A_k(h,x)) = \P\left( \bigcup_{j\ge 1}
    \set{n_\sw^k+1\ge j, a^\star\notin\set{a^1,a^2,\dots,a^{j-1}},
    a^\star=a^j}\right) \\
  & = \sum_{j\ge 1} \P(n_\sw^k+1\ge j) \cdot
    \P(a^\star\notin\set{a^1,a^2,\dots,a^{j-1}}, 
    a^\star=a^j \mid n_\sw^k+1\ge j).
\end{align*}
Now, suppose we know that $n_\sw^k+1\ge j$, then the algorithm have
seen the reward on $a^1,\dots,a^{j-1}$. By the uniform prior of
$a^\star$, if the algorithm has observed the rewards for all $a\in S$
and found that $a^\star\notin S$, the corresponding posterior for
$a^\star$ would be uniform on $[A]\setminus S$. Therefore, we have
recursively that
\begin{equation*}
  \P(a^\star\notin \set{a^1,\dots,a^{j-1}}, a^\star=a^j \mid
  n_\sw^k+1\ge j) = \prod_{\ell=1}^{j-1} \frac{A-\ell}{A-\ell+1} \cdot
  \frac{1}{A-j+1} = \frac{1}{A}.
\end{equation*}
Substituting this into the preceding bound gives
\begin{equation*}
  \Phi_k(h, x) = \frac{1}{A} \sum_{j\ge 1} \P(n_\sw^k+1\ge
  j) = \frac{\E[n_\sw^k+1]}{A}
\end{equation*}
and thus
\begin{equation*}
  \E_{a^\star, M_{a^\star}}[V_1^{\pi_k}] \le \frac{1}{S}\sum_{h, x}
  \Phi_k(h,x) \le \frac{1}{S}\sum_{h, x} \frac{\E[n_\sw^k(h,x)+1]}{A}
  \le \frac{H}{A} + \frac{\E[N_\sw^k]}{SA}
\end{equation*}
As $N_\sw^k\le N_\sw^K\le HSA/2$ almost surely, we have for all $k$ that
\begin{equation*}
  \E_{a^\star, M_{a^\star}}[V_1^{\pi_k}] \le H/A + H/2
  \le 3H/4
\end{equation*}
when $A\ge 4$ and thus the regret can be lower bounded as
\begin{equation*}
  KH - \sum_{k=1}^K \E_{a^\star, M_{a^\star}}[V_1^{\pi_k}] \ge KH/4,
\end{equation*}
concluding the proof. \qed

\end{document}